\documentclass[twoside]{article}

\usepackage[preprint]{aistats2026}
\usepackage[round]{natbib}

\usepackage{xcolor}
\usepackage{enumitem}
\usepackage{algorithm}
\usepackage{algcompatible}
\usepackage{amsthm}
\usepackage{comment}
\usepackage{amsfonts}
\usepackage{amssymb}
\usepackage{changepage}
\usepackage{graphicx}
\usepackage{float}
\usepackage{subcaption}
\usepackage{makecell}
\usepackage{booktabs}
\usepackage{capt-of}
\usepackage{subcaption}  
\usepackage{tabularx}
\usepackage{hyperref}

\theoremstyle{plain}
\newtheorem{theorem}{Theorem}[section]
\newtheorem{proposition}[theorem]{Proposition}
\newtheorem{lemma}[theorem]{Lemma}
\newtheorem{claim}[theorem]{Claim}

\theoremstyle{definition}

\newtheorem{assumption}[theorem]{Assumption}
\theoremstyle{remark}
\newtheorem{remark}[theorem]{Remark}

\usepackage{multirow}
\usepackage{tikz}
\usetikzlibrary{automata}
\usetikzlibrary{arrows.meta, positioning}

\renewcommand{\sp}{\mathrm{sp}}
\newcommand{\tr}{\top}
\newcommand{\QS}{Q_*}

\newcommand{\bE}{\mathbb{E}}
\newcommand{\bI}{\mathbb{I}}
\newcommand{\bR}{\mathbb{R}}
\newcommand{\ap}{a'}

\newcommand{\stp}{s'}
\newcommand{\mup}{\mu'}
\newcommand{\mub}{\bar{\mu}}

\newcommand{\alS}{\alpha^*}

\newcommand{\CRPI}{CRPI}

\newcommand{\cA}{\mathcal{A}}

\newcommand{\cM}{\mathcal{M}}
\newcommand{\cP}{\mathcal{P}}

\newcommand{\cS}{\mathcal{S}}
\newcommand{\cU}{\mathcal{U}}
\newcommand{\cF}{\mathcal{F}}

\newcommand{\gep}{\geq_{\textnormal{p}}}
\usepackage{pifont}

\renewcommand{\eqref}[1]{(\ref{#1})}

\renewcommand{\thefootnote}{\fnsymbol{footnote}}
\newcommand{\ones}{\mathbf{1}}

\newcommand{\rpiddpg}{RPI\textsubscript{DDPG}}

\begin{document}

%
\runningtitle{Monotone and Conservative PI
Beyond the Tabular Case}

%

\onecolumn
\aistatstitle{Monotone and Conservative Policy Iteration \\ Beyond the Tabular Case\footnotemark \footnotetext{Accepted at the International Conference on Artificial Intelligence and Statistics (AISTATS 2026), Tangier, Morocco.}}



\begin{center}
  
  {\large
    S.R. Eshwar$^{1}$ \quad
    Gugan Thoppe$^{1}$ \quad
    Ananyabrata Barua$^{1}$\footnotemark \quad Aditya Gopalan$^{1}$ \quad Gal Dalal$^{2}$
    \par}
  \vspace{1em} 

  {\normalsize
    $^{1}$Indian Institute of Science, Bengaluru, India \\
    $^{2}$NVIDIA Research, Israel \\
    \texttt{\{eshwarsr, gthoppe, ananyabratab, aditya\}@iisc.ac.in, galdl20@gmail.com}
    \par}
  \vspace{2em} 
\end{center}
\footnotetext{Pre-doc supported by Walmart Centre for Tech Excellence at IISc.}
\renewcommand{\thefootnote}{\arabic{footnote}}
\setcounter{footnote}{0} 

\begin{abstract}
    We introduce Reliable Policy Iteration (RPI) and Conservative RPI (\CRPI), variants of Policy Iteration (PI) and Conservative PI (CPI), that retain tabular guarantees under function approximation. RPI uses a novel Bellman-constrained optimization for policy evaluation. We show that RPI restores the textbook \textit{monotonicity} of value estimates and that these estimates provably \textit{lower-bound} the true return; moreover, their limit partially satisfies the \textit{unprojected} Bellman equation. \CRPI\ shares RPI's evaluation, but updates policies conservatively by  maximizing a new performance-difference \textit{lower bound} that explicitly accounts for function-approximation-induced errors. \CRPI\ inherits RPI's guarantees and, crucially, admits per-step improvement bounds. In initial simulations, RPI and \CRPI\ outperform PI and its variants. Our work addresses a foundational gap in RL: popular algorithms such as TRPO and PPO derive from tabular CPI yet are deployed with function approximation, where CPI's guarantees often fail-leading to divergence, oscillations, or convergence to suboptimal policies. By restoring PI/CPI-style guarantees for \textit{arbitrary} function classes, RPI and \CRPI\ provide a principled basis for next-generation RL.
\end{abstract}

\section{INTRODUCTION}
Function Approximation (FA) in Reinforcement Learning (RL) creates a fundamental challenge: while necessary for handling large state-action spaces, it also results in issues such as divergence, policy and value oscillations, training instability, and convergence to sub-optimal policies---sometimes even to the worst policy \citep{aethelios2025beyond, patterson2024empirical, baird1995residual, young2020understanding, henderson2018deep, gopalan2025does}. A closer look at these issues reveals a gap between practice and theory: while mainstream RL algorithms are deployed with (possibly high-capacity) FA, the associated guarantees apply only in tabular or near-tabular regimes \citep{bertsekas1996neuro, kakade2002approximately, hasselt2010double, haarnoja2018soft, metelli2021safe}. 

The roots of this gap stem from the fact that the textbook (model-based) Policy Iteration (PI) method \citep{howard1960dynamic}---the foundation of all actor–critic-type approaches---lacks a FA variant that preserves its core guarantees. Recall that PI alternates between policy evaluation and policy update. In tabular settings, the value functions of the successive policies improve \textit{monotonically} and \textit{converge} to the optimum. With FA, however, even in the model-based setting---or with access to infinite data---approximation errors during policy evaluation can corrupt the policy updates, causing  divergence, oscillations, and learning instability \citep{bertsekas2011approximate}. Therefore, it is not surprising that these pathologies also pervade model-free actor–critic methods wherein we  must also contend with the epistemic uncertainty of working with only sampled data  \citep{thrun1993issues,van2016deep,td3}. In a similar vein, Conservative PI (CPI) \citep{kakade2002approximately} and its extension Safe PI (SPI) \citep{metelli2021safe} currently lack FA variants that retain their per-step improvement guarantees. It has been noted that PI, CPI, and SPI differ markedly in how they respond to approximation errors during policy updates: PI often learns quickly but can stall or oscillate, whereas CPI and SPI make more measured updates and, in some domains, reach better policies \citep{metelli2021safe}. 

Together, these observations raise two central questions: (i) Can FA variants of PI be built without losing value-estimate monotonicity? (ii) Can FA variants of CPI and SPI be designed while retaining their per-step improvement guarantees? These questions bear directly on widely used value-based and policy-optimization methods, including DQN \citep{dqn}, TRPO \citep{trpo}, and PPO \citep{schulman2017proximal}, and are therefore central to placing modern RL on firmer theoretical foundations.

We answer both questions in the affirmative. Our work's key contributions are as follows:
\begin{enumerate}[leftmargin=*]

    \item \textbf{Algorithms:} We introduce Reliable Policy Iteration (RPI) and its conservative variant, \CRPI, in Algorithms~\ref{alg:RPI} and \ref{alg:CRPI}. Both use a novel policy-evaluation rule: compute the value estimate \emph{farthest} from the previous one subject to a linear Bellman-inequality constraint. The two methods differ in the policy-update step: RPI uses the standard greedy update, whereas \CRPI\ uses a conservative one akin to CPI/SPI. Unlike PI, CPI, and SPI, our methods apply under general FA.
    
    \item \textbf{RPI Theoretical Guarantees:} We show that RPI extends the classical PI guarantees of per-step improvement and convergence to arbitrary FA (Section~\ref{subsec:RPI}), where such guarantees were previously out of reach. Specifically, we show that RPI's value estimates are non-decreasing, lower bound the true policy values, and converge to a vector that partially satisfies the \textit{unprojected} Bellman equation. We also bound the performance gap between RPI’s terminal policy and the optimal one. Finally, we show that RPI generalizes classical PI and that its policy-evaluation step admits a constrained projection interpretation under an $\ell_1$-type norm.

    \item \textbf{\CRPI\ Theoretical Guarantees:} We develop a generalization of the performance-difference lemma to arbitrary FA (Section~\ref{subsec:CRPI}) that (i) accommodates FA-based estimates of the true advantage function and (ii) explicitly accounts for approximation errors. Building on this, we show that \CRPI’s policy update maximizes a lower bound on the performance gap, yielding the first per-step improvement guarantees under FA. We further prove that \CRPI’s suboptimality converges at $O(1/k)$ rate to an $\epsilon$-neighborhood of its limiting value.
    
    \item \textbf{Simulations}: In model-based inventory control with dense rewards, RPI outperforms approximate variants of PI, CPI, and SPI in both learning speed and terminal policy quality. By contrast, in chain walk with sparse rewards, \CRPI\ learns more conservatively but often identifies better terminal policies. In deep RL experiments on MuJoCo robotic tasks, RPI-based methods are either competitive with or outperform standard baselines, and their critic estimates often display near-monotone behavior while frequently lower bounding the Monte Carlo returns.
\end{enumerate}

\section{RELATED WORK}
\label{s:related.work}
The earliest use of FA can be traced to Samuel’s checkers program \citep{samuel1959some,samuel1967some}. It selected moves via multistage lookahead while evaluating positions with a value function expressed as a linear combination of handcrafted features. Since then, many attempts have been made to extend PI to the FA setting, which can be categorized and summarized as follows.

\textbf{PI with Approximate Evaluation.} A large body of work modifies only the policy-evaluation step, while leaving the greedy policy update unchanged. These methods typically follow one of four approaches: (i) minimizing the mean-squared projection error, as in TD(1) \citep{tsitsiklis1996analysis}; (ii) solving the projected Bellman fixed-point equation, as in LSPI and TD($\lambda$)  \citep{lagoudakis2003least,tsitsiklis1996analysis}; (iii) minimizing the Bellman residual error, e.g.,  \citep{scherrer2010should}; or (iv) minimizing the Bellman backup error, as in Fitted Q-Iteration and AMPI-Q \citep{ernst2005tree,dqn,ampi}. However, to paraphrase \citet{bertsekas2011approximate}, these methods presume that a more accurate value approximation should yield a better policy---a plausible but far from self-evident assumption. In practice, they diverge, oscillate among worse policies, or otherwise behave unreliably \citep{bertsekas2011approximate,patterson2024empirical,gopalan2025does,young2020understanding}. A theoretically sound resolution of these issues has remained elusive so far. 

\textbf{PI with Conservative Policy Updates.} A complementary line of work makes policy updates cautious, either by constraining the KL divergence between successive policies or by interpolating between the current policy and a greedy one. This idea goes back to conservative PI \citep{kakade2002approximately} and is further developed in safe PI \citep{metelli2021safe}; it is also reflected in influential methods such as TRPO \citep{trpo} and PPO \citep{schulman2017proximal}. Such cautious updates can yield strong per-step improvement bounds \citep[Remark~3]{metelli2021safe}, and empirically they often learn more conservatively yet attain better terminal policies \citep[Figures~9,~13]{metelli2021safe}. However, their theory relies on accurate value estimates for essentially every state--action pair, an assumption that breaks down under FA.

\textbf{Policy Iteration with Multi-step Lookahead.} In classical PI, the actor jumps to a policy that is 
greedy with respect to the current policy's value function. Multistep-lookahead variants instead choose a policy that is greedy over an 
$h$-step rollout. While this still ensures monotonic improvement in tabular MDPs \citep{efroni2018beyond}, the picture becomes bleak once FA enters. 
\cite{winnicki2021role} show that 
with least squares approximation during evaluation, the value estimates can diverge unless $h$ exceeds a problem-dependent threshold. In fact, lookahead faces inherent limitations even in tabular settings: partial policy evaluation 
may lead to divergence \citep{efroni2019combine}, and conservative PI loses its one-step monotonicity as soon as $h > 1$ \citep{efroni2018multiple}. 
These results point to a need to redesign policy evaluation and policy update to get better guarantees. 

\textbf{PI-Adjacent Alternatives.} There is also classification-based PI \citep{lazaric2010analysis} which treats policy-update as a supervised learning problem, but forfeits per-iteration guarantees. There is also linear-programming-based approaches for dynamic programming with linear FA \citep{de2003linear}. Unlike RPI, they are single-shot methods that directly seek a lower bound on the optimal value. 

\textbf{Policy-Optimization Beyond PI.} While several recent works study policy-gradient, trust-region, and mirror-descent methods under FA, their guarantees differ fundamentally from those considered here. These methods typically rely on realizability assumptions or uniform bounds on critic approximation error across all intermediate policies. For example, \citet{liu2019neural} assume that the function class contains $Q_\pi$ for every policy $\pi$, while \citet{agarwal2021theory} and \citet{alfano2023novel} assume uniform bounds on critic approximation error across iterations. Consequently, their guarantees bound quantities such as the minimum or average suboptimality over $T$ iterations in terms of the approximation-error constant (e.g., $\epsilon_{\text{bias}}$ or $\epsilon_{\text{approx}}$). When this constant is large, the resulting bounds can be large and therefore do not preclude oscillations or convergence to suboptimal policies.

\section{PROBLEM FORMULATION, PROPOSED ALGORITHMS, AND MAIN RESULTS}
\label{s:setup.algoroithm.results}
We formalize our setup and goals in Section~\ref{subsec:setup_and_problem}, and then present RPI and \CRPI---along with their theoretical guarantees---in Sections~\ref{subsec:RPI} and \ref{subsec:CRPI}, respectively. Appendix \ref{s:apdx.Proofs} provides detailed proofs of all our results.

\subsection{Setup and Problem Formulation}
\label{subsec:setup_and_problem}
We have a stationary MDP $\cM \equiv (\cS, \cA, \cP, r, \gamma).$ Here, $\cS$ and $\cA$ are finite state and action spaces, respectively, with $|\cS| := S$ and $|\cA| := A.$ Further, $\cP$ is the transition kernel and $\cP(s'|s, a)$ specifies the probability of reaching state $s'$ from state $s$ under action $a.$ Finally, $r: \cS \times \cA \to \bR$ is the per-step reward function, and $\gamma \in [0, 1)$ the discount factor. 

For any set $\cU,$ let $\Delta(\cU)$ be the set of distributions on it. Now, for a stationary policy $\mu: \cS \to \Delta(\cA),$ define its Q-value function $Q_\mu : \cS \times \cA \to \bR$ by $Q_\mu(s, a) := \bE \left[\sum_{t = 0}^\infty \gamma^t r(s_t, a_t) \middle|s_0 = s, a_0 = a \right],$ where $s_{t + 1} \sim \cP(\cdot|s_t, a_t)$ and $a_{t + 1} \sim \mu(\cdot|s_{t + 1})$ for all $t \geq 0.$ Further, let $T_\mu: \bR^{SA} \to \bR^{SA}$ and $T: \bR^{SA} \to \bR^{SA}$ be the Bellman operators given by
\begin{align*}
    T_\mu Q (s, a) = {} & r(s, a) + \gamma \sum_{s', a'} \cP(s'|s, a) \mu(a'|s') Q(s', a') \\[-2ex]
    \intertext{and}
    T Q(s, a) = {} & r(s, a) + \gamma \sum_{s' \in \cS} \cP(s'|s, a) \max_{a'} Q(s', a').
\end{align*}
    
Our goal here is twofold: (i) modify tabular PI’s policy-evaluation step to restore its monotonicity and convergence guarantees under FA; and (ii) adapt the policy-update step to maximize an FA-based performance-improvement gap, akin to CPI and SPI.

We use $\cF = \{f: \cS \times \cA \to \bR\}$ to denote a given FA space for representing the Q-value functions. Clearly, each $f \in \cF$ can also be interpreted as a vector in $\bR^{SA}.$ 

\subsection{Reliable Policy Iteration (RPI)}
\label{subsec:RPI}

\begin{algorithm}[t!]
   \caption{Reliable Policy Iteration (RPI)}
   \label{alg:RPI}
\begin{algorithmic}
   \STATE {\bfseries Input:} FA class $\cF,$ policy $\mu_0,$ an initial approximation $f_0 \in \cF$ of $Q_{\mu_0},$ and a norm $\|\cdot\|$ 
   \FOR{$k = 0, 1, 2$ ... until convergence}
   \STATE \textbf{Policy Evaluation}:
    \begin{equation}
    \label{e:GUIDE.PE.Opt}
        \begin{aligned}
            f_{k+1} \in  \underset{f \in \cF}{\arg\max} \quad & \|f - f_k\| \\
            \text{s.t.} \quad & T_{\mu_k} f  \geq f \geq f_k
        \end{aligned}
    \end{equation}
    \STATE \textbf{Policy Improvement}: 
    \vspace{-1em}
    \begin{equation}
        \mu_{k + 1} \in \{\mu: \mu \text{ is a deterministic } 
        \text{ greedy policy w.r.t. } f_{k + 1} \}
    \end{equation}
    \ENDFOR
\end{algorithmic}
\end{algorithm}

RPI's description is given in Algorithm~\ref{alg:RPI}. The inequalities there are  coordinate wise and we retain this convention for vector inequalities throughout. Like other PI implementations, RPI also interleaves policy evaluations and policy updates, but its novelty lies in the evaluation step. Rather than seeking a close approximation to $Q_{\mu_k},$ or minimizing the projected-Bellman error, or even being conservatively close to $f_k,$ RPI chooses $f \in \cF$ that is farthest away from $f_k$ under a given norm $\|\cdot\|,$ subject to $T_{\mu_k} f \geq f \geq f_k.$ The policy-update step is standard: it replaces the current policy $\mu_k$ with the one that is greedy with respect to $f_{k + 1},$ as in classical PI and in existing FA variants such as approximate PI \citep{bertsekas1996neuro} and AMPI-Q \citep{ampi}. 

%
We now state RPI's performance guarantees under general FA, covering both linear and non-linear settings. For any two vectors $Q, Q' \in \bR^{SA},$ we write $Q \gep Q'$ to imply that $Q(s, a) \geq Q'(s, a)$ for all $(s, a),$ with equality holding on \textit{at least one coordinate}.

\begin{theorem}[\textbf{RPI properties with general FA}]
    \label{thm:RPI.FA} 
    Suppose the FA space $\cF$ is a closed subset of $\bR^{SA}$ and the initial policy and value estimates satisfy $T_{\mu_0} f_0 \geq f_0.$ Then, the following claims hold.
    \begin{enumerate}[leftmargin=*]
        \item For any $k \geq 0,$ $f_k$ satisfies the constraints in (\ref{e:GUIDE.PE.Opt}); hence, a solution to \eqref{e:GUIDE.PE.Opt} always exists. 

        \item $(f_k)_{k \geq 0}$ is 
        non-decreasing and $Q_{\mu_k} \geq f_k$  $ \forall k \geq 0.$ 
                
        \item $f_\infty := \lim_{k \to \infty} f_k$ exists and satisfies $T_{\mu_\infty} f_\infty = T f_\infty \geq f_\infty,$ where $\mu_\infty$ is any policy that is greedy with respect to $f_\infty.$ Furthermore, if $\QS$ denotes the optimal Q-value function, then
        \begin{align*}
            \|Q_{\mu_\infty} - \QS\|_\infty \leq {} & \frac{\|T f_\infty - f_\infty\|_\infty}{1 - \gamma} 
            = {} \frac{\|T_{\mu_\infty} f_\infty - f_\infty\|_\infty}{1 - \gamma}.
        \end{align*}

        \item Additionally suppose $\|\cdot\|$ is strictly monotone: $Q \geq Q' \geq 0$ and $Q \neq Q'$ imply $\|Q\| > \|Q'\|.$ Also suppose the function class $\cF$ has room to improve in the positive orthant centered at $f_\infty$: there is a $\delta_0 > 0$ such that, for every $0 < \delta \leq \delta_0,$ the function class $\cF$ contains at least one $f$ satisfying $\|f - f_\infty\| < \delta$ and $f > f_\infty$ (coordinate-wise). Then, $T_{\mu_\infty} f_\infty = T f_\infty \gep f_\infty,$ i.e., $f_\infty$ partially satisfies the Bellman and Bellman optimality equations.
    \end{enumerate}
\end{theorem}

We next show that RPI mimics PI in the tabular case. Thus, it is a true generalization of PI.

\begin{proposition}[\textbf{RPI generalizes PI}]
\label{prop:RPI.PI.generalization}
Suppose $\cF = \bR^{SA},$ $T_{\mu_0} f_0 \geq f_0,$ and the norm $\|\cdot\|$ is strictly monotone (as defined in Theorem~\ref{thm:RPI.FA}). Then, $f_{k + 1} = Q_{\mu_k}$ $\forall k \geq 0,$ $f_\infty = \QS,$ and $\mu_\infty$ is  optimal.  
\end{proposition}

Our next result shows that RPI's evaluation step has a projection interpretation under $\ell_1$-type norms. 

\begin{proposition}[\textbf{Projection view under $\ell_1$-type-norm}]
\label{prop:RPI.l1.projection}
Suppose $\cF$ is closed, $T_{\mu_0} f_0 \geq f_0,$ and the norm in \eqref{e:GUIDE.PE.Opt} is some $w$-weighted $\ell_1$-norm $\|\cdot\|_{w, 1}$. That is, let $w \in \bR^{SA}$ be made up of strictly positive values and $\|f\|_{w, 1} := \sum_{s, a} w(s, a) |f(s, a)|.$ Then, $f_{k + 1} \in \underset{f \in \cF: T_{\mu_k} f \geq f \geq f_k}{\arg\min} \|f - Q_{\mu_k}\|_{w, 1}.$ 
\end{proposition}

\textbf{Discussion}: We highlight the following three key implications of the above results.
\begin{enumerate}[leftmargin=*]
    \item \textbf{Monotonic reliability under FA}. Theorem~\ref{thm:RPI.FA} shows that $(f_k)$ is coordinate-wise non-decreasing and lower bounds the true Q-values of the corresponding policies. Even if the policy estimate  degrades, i.e., $Q_{\mu_k}$ decreases, the drop in performance will not go below $f_k,$ which is non-decreasing. RPI is the first FA-variant of PI that comes with a monotonic improvement guarantee for its value estimates. Recall that existing methods like PI \citep{howard1960dynamic}, CPI \citep{kakade2002approximately}, or SPI \citep{metelli2021safe} require accurate $Q$-value estimates across all $SA$-many state-action pairs to provide such guarantees. 
    
    \item \textbf{Convergence to Bellman-consistent points}. Theorem~\ref{thm:RPI.FA} also shows that $f_k$ converges to a limit $f_\infty$ that \textit{partially} satisfies both the optimality Bellman equation and the Bellman equation for $\mu_\infty$. While multiple such fixed points may exist---a fundamental limitation of FA---RPI nevertheless remains aligned with the central RL objective of solving the unprojected Bellman equations. In contrast, projection-based schemes typically target projected surrogates instead, and can oscillate between poor policies \citep{bertsekas2011approximate}.

    \item \textbf{Geometric interpretation in weighted-$\ell_1$}. Proposition~\ref{prop:RPI.l1.projection} shows that, under any weighted-$\ell_1$ norm, RPI's evaluation step is equivalent to projecting $Q_{\mu_k}$ onto the constrained set $\{f \in \cF: T_{\mu_k}f \geq f \geq f_k\}$. Unlike standard full-space projections, this projection incorporates the Bellman inequalities directly into the feasible set. Although this exact geometric interpretation need not extend to arbitrary norms, the coordinate-wise monotonicity guarantees continue to hold for any norm. Thus, our constrained optimization retains the key monotonic structure of tabular PI even in the FA setting.
\end{enumerate}

In summary, by embedding the Bellman inequalities directly into the evaluation program in \eqref{e:GUIDE.PE.Opt}, RPI fully restores the classical PI guarantees.

\begin{algorithm}[t!]
   \caption{Conservative RPI (\CRPI)}
   \label{alg:CRPI}
\begin{algorithmic}
   \STATE {\bfseries Input:} FA class $\cF,$ initial policy $\mu_0,$ an initial approximation $f_0 \in \cF$ of $Q_{\mu_0},$ distribution $\nu$ for sampling the initial $(s, a)$-pair, and a norm $\|\cdot\|$ 
   \FOR{$k = 0, 1, 2$ ... until convergence}
   \STATE \textbf{Policy Evaluation}:
    \begin{equation}
    \label{e:CRPI.PE.Opt}
        \begin{aligned}
            f_{k+1} \in  \underset{f \in \cF}{\arg\max} \quad & \|f - f_k\| \\
            \text{s.t.} \quad & T_{\mu_k} f  \geq f \geq f_k
        \end{aligned}
    \end{equation}
    \STATE \textbf{Policy Improvement}: 
    \begin{equation}
        \begin{aligned}
            \mub_{k} \in {} & \{\mu: \mu \text{ is a greedy policy w.r.t. } f_{k + 1} \} \\
            \mu_{k + 1} \leftarrow {} & \alpha_k \mub_k + (1 - \alpha_k) \mu_k,
        \end{aligned}
    \end{equation}
    where $\alpha_k = \min\{1, \alS(f_{k + 1}; \mu_k, \mub_k)\}$ and 
    $\alS(f; \mu, \mub)$ is as defined in \eqref{e:alS.defn}. 
    %
    %
    \ENDFOR
\end{algorithmic}
\end{algorithm}

\subsection{Conservative RPI (\CRPI)}
\label{subsec:CRPI}
While vanilla PI uses greedy policy updates, CPI and SPI employ conservative updates instead. Under FA and sampling noise, such CPI/SPI-style updates have empirically been shown to produce better terminal performance \citep{metelli2021safe}. In addition, in tabular or near-tabular settings, CPI/SPI admit per-iteration policy-improvement guarantees by maximizing a lower bound on the performance-difference gap. This section shows how to bring such conservative policy updates into the RPI framework.

Algorithm~\ref{alg:CRPI} describes \CRPI. It keeps RPI’s policy-evaluation step but replaces the greedy policy update with a conservative one: $\mu_{k+1}$ is chosen as a convex combination of $\mu_k$ and the policy greedy with respect to $f_{k+1}$. For $\alpha<1$, Theorem~\ref{thm:CRPI.FA} shows that this update maximizes suitable lower bounds on an \textit{approximate performance-difference gap}.

We now formally define the various notations used in Algorithm~\ref{alg:CRPI}. For a policy $\mu,$  
$P_{\mu}$ is the $SA \times SA$ matrix, whose $((s, a), (\stp, \ap))$-th entry is $P(\stp|s, a) \mu(\ap|\stp).$ We use $\nu \in \bR^{SA}$ for an arbitrary but fixed initial distribution on the $\cS \times \cA$ space, and $d_{\mu}^\tr := (1 - \gamma) \nu^\tr[\bI - \gamma P_{\mu}]^{-1}$ for the $SA$-dimensional discounted state–action occupancy measure associated with $\mu.$ For a policy $\mu,$ $\delta_{\mu} := d_{\mu}^\tr P.$ Further, for policies $\mu, \mup$ and  vector $f \in \bR^{SA},$
\begin{equation}
    \label{e:a.mu.mup.defn}
    a_{\mu}^{\mup}(f) := [P_{\mup}- P_{\mu}]\,f \; \text{ and } \;     A_{\mu}^{\mup}(f) := d_{\mu}^\tr  a_{\mu}^{\mup}(f).
\end{equation}
Note that the $(s, a)$-th coordinate of $a_{\mu}^{\mup}(f),$ i.e.,  
\begin{equation}
    \label{e:a.mu.mup.s.a.defn}
    a_{\mu}^{\mup}(f)(s, a) = \sum_{s', \ap} P(\stp|s,a)  \mup(a'|s')
    \times \Bigl(
    f(s',a') - \langle \mu(\cdot|s'),\, f(s',\cdot)\rangle \Bigr),
\end{equation}
where $\langle \cdot, \cdot \rangle$ denotes the inner product. Thus, if $f$ is an estimate of $Q_{\mu},$ then $a_{\mu}^{\mup}(f)$ approximates the advantage function of $\mu,$ relative to $\mup.$ Finally, for $x \in \bR^{SA},$ 
$\sp(x) := \max_{s, a} x(s, a) - \min_{s, a} x(s, a)$ denotes the span semi-norm of $x$, and $\|\mu_1 - \mu_2\|_{1, \delta_{\mu}} := \sum_{s} \delta_\mu(s)\  \|\mu_1(\cdot|s) - \mu_2(\cdot|s)\|_1.$ 

We now present our main results on \CRPI. The first gives a FA generalization of the performance-difference lemma---the backbone of CPI, SPI, TRPO, and PPO.

\begin{lemma}[Approximate Performance-Difference Lemma]
\label{lem:Approx.Performance.Diff.Lemma}
Suppose $\mu$ and $\mup$ are arbitrary stationary policies and $f \in \bR^{SA}$ is a arbitrary vector such that $T_{\mu} f \geq f.$ Then, for any $H \geq 0,$
%
%
%
\begin{equation}
\label{e:approx.Perf.Diff.Lemma.inequality}
    \nu^\tr Q_{\mup} - \nu^\tr f 
    \geq  \frac{\gamma}{1 - \gamma} d_{\mup}^\tr  a_{\mu}^{\mup}(f)  +   \sum_{h = 0}^H \gamma^h \nu^\tr P_{\mup}^h [T_{\mu}f - f].
\end{equation}
\end{lemma}

\begin{remark}
    \textbf{Comparison with \citep{kakade2002approximately}.} The original performance-difference lemma (in the Q-value-function version) states that 
    \(
        \nu^\tr Q_{\mup} - \nu^\tr Q_{\mu} = \frac{\gamma}{1 - \gamma} d_{\mup}^\tr a_{\mu}^{\mup}( Q_{\mu}).
    \)
    In tabular settings, where $Q_\mu=f$ and hence $T_\mu f=f$, our bound in \eqref{e:approx.Perf.Diff.Lemma.inequality} reduces to the above, since the second term vanishes. Our result's main benefit is that it also applies in practical FA regimes, where $Q_\mu$ is not known exactly. In particular, for any $f$ satisfying $T_\mu f \geq f$---and therefore $Q_\mu \geq f$ by Claim~\ref{c:Q.mu.lower.Bd}---our result yields a computable lower bound on an approximate performance gap between $\mup$ and $\mu$. We say approximate since we use $\nu^\tr f$ in place of $\nu^\tr Q_\mu.$ Our bound also contains additional error terms involving $T_\mu f - f$, which arise purely due to FA, and these terms become smaller as the truncation horizon $H$ increases.
\end{remark}

Next, for mixture policies, we derive two performance-gap bounds, each quadratic in the mixing parameter $\alpha$. These follow by using \eqref{e:approx.Perf.Diff.Lemma.inequality} with $H = 0$ and $H = 1.$ Additional polynomial bounds follow by using higher values of $H,$ but we do not pursue this here.

\begin{proposition}
\label{prop:perf.Diff.Quad.Bd}
    Let $\mu$ and $\mub$ be arbitrary stationary policies and $f \in \bR^{SA}$ be an arbitrary vector such that $T_{\mu} f \geq f.$  Then, for any $\alpha \in [0, 1]$ and the mixture policy $\mup = \alpha \mub + (1 - \alpha) \mu,$ we have
    \begin{equation}
        \nu^\tr  Q_{\mup}  - \nu^\tr f \geq \Psi_1(\alpha)  \geq \Psi_0(\alpha),
    \end{equation}
    where 
    \begin{align}
        \Psi_1(\alpha) = {} & \Psi_1(\alpha; f,  \mu, \mub) \nonumber \\
        = {} & - \frac{ \alpha^2 \gamma^2}{2  (1 - \gamma)^2} \|\mub - \mu\|_{1, \delta_\mu}\ \sp(a_\mu^{\mub} (f))  + \alpha\! \left[\frac{\gamma}{1 - \gamma} A_{\mu}^{\mub} (f) + \gamma \nu^\tr a_{\mu}^{\mub}(T_{\mu}f - f) \right] +  \nu^\tr(\bI + \gamma P_{\mu})[T_{\mu}f - f]  \label{e:perf.Diff.Quad.Bd.H.equals.1} \\
        \intertext{and}
        \Psi_0 (\alpha) = {} & \Psi_0(\alpha; f, \mu, \mub) \nonumber \\
        = {} & - \frac{ \alpha^2 \gamma^2}{2  (1 - \gamma)^2} \|\mub - \mu\|_{1, \delta_\mu}\ \sp(a_\mu^{\mub} (f))  + \frac{ \alpha\gamma}{1 - \gamma} A_{\mu}^{\mub} (f) + \nu^\tr [T_{\mu}f - f]. \label{e:perf.Diff.Quad.Bd}
    \end{align}

\end{proposition}
\begin{remark} 
    \textbf{Comparison with \citep{metelli2021safe}.}
   Our bound in \eqref{e:perf.Diff.Quad.Bd} generalizes equation (P.6) of \citep{metelli2021safe}: (i) it uses $A_{\mu}^{\mub}(f)$ and $a_{\mu}^{\mub}(f)$ instead of the true  advantage function estimates and (ii) it introduces an additional $(T_{\mu} f - f)$ term that captures FA error. Moreover, the specialization \eqref{e:perf.Diff.Quad.Bd.H.equals.1} dominates \eqref{e:perf.Diff.Quad.Bd} (i.e., is pointwise larger) and incorporates FA-dependent coefficients, yielding extra FA-specific guidance for choosing the mixture parameter.
\end{remark}

Our next result provides a per-step improvement guarantee for \CRPI\  under general FA. 

For arbitrary stochastic policies $\mu$ and $\mub$ and a arbitrary vector $f \in \bR^{SA},$ let
\begin{equation}
    \begin{aligned}
        \alS_1 \equiv \alS_1(f, \mu, \mub) := \underset{\alpha \in \bR}{\arg\max}  \Psi_1(\alpha)      = \frac{\eta_1 + \eta_2}{\partial_1}
    \end{aligned}
\end{equation}
and
\begin{equation}
    \begin{aligned}
        \alS_0 \equiv \alS_0(f, \mu, \mub) := {} & \underset{\alpha \in \bR}{\arg\max}\,  \Psi_0(\alpha) = \frac{\eta_1}{\partial_1},\hspace{-1.5em}
    \end{aligned}
\end{equation}
where $\eta_1 \equiv \eta_1(f, \mu, \mub) = (1 - \gamma)  A_{\mu}^{\mub}(f),$ $\eta_2 \equiv \eta_2(f, \mu, \mub) = (1 - \gamma)^2 \nu^\tr a_{\mu}^{\mub}(T_{\mu}f - f),$ and $\partial_1 \equiv \partial_1(f, \mu, \mub) = \gamma \|\mub - \mu\|_{1, \delta_{\mu}}\  \sp(a_{\mu}^{\mub} (f)).$
Finally, let 
\begin{equation}
    \label{e:alS.defn}
    \alS \equiv \alS(f, \mu, \mub) = 
    \begin{cases}
        \alS_1 & \text{if } \alS_1 > 0, \\
        \alS_0 & \text{otherwise}.
    \end{cases}
\end{equation}

\begin{theorem}
\label{thm:perf.bds.FA}
    Let $\mu$ be an arbitrary stochastic policy and $f \in \bR^{SA}$ an arbitrary vector such that $T_{\mu} f \geq f.$ Also, let $\mub$ be a greedy policy with respect to $f$ and $\nu$ an arbitrary initial distribution on  $\cS \times \cA$. Then, $\alS_0 \geq 0$ and the following bounds hold for the mixture policy $\mup = \alpha \mub + (1 - \alpha) \mu$ where $\alpha = \min\{1, \alS\}.$
    \begin{enumerate}
        \item If $\alS_1 > 1,$ then $\nu^\tr Q_{\mup} - \nu^\tr f \geq \Psi_1(1).$

        \item If $\alS_1 \in [0, 1],$ then $\nu^\tr Q_{\mup} - \nu^\tr f \geq \Psi_1(\alS_1).$

        \item If $\alS_1 < 0$ and $\alS_0 > 1,$ then $\nu^\tr Q_{\mup} - \nu^\tr f \geq \Psi_0(1).$

        \item If $\alS_1 < 0$\! and $\alS_0 \leq 1,$\! then $\nu^\tr Q_{\mup} - \nu^\tr f \geq \Psi_0(\alS_0).$
    \end{enumerate}
\end{theorem}

\begin{theorem}
\label{thm:CRPI.FA}
Suppose \(\cF\subseteq \bR^{SA}\) is closed and the initial pair \((\mu_0,f_0)\) satisfies \(T_{\mu_0}f_0 \ge f_0\). Then all conclusions of Theorem~\ref{thm:RPI.FA} hold for the \CRPI\ sequence \((f_k,\mu_k)_{k\ge 0}\). In particular, for every \(k\ge 0\), $T_{\mu_k} f_{k+1} \ge f_{k+1}.$
Consequently, Theorem~\ref{thm:perf.bds.FA} applies to the gap $\nu^\tr Q_{\mu_{k+1}} - \nu^\tr f_{k+1}$ for every \(k\ge 0\) and every initial distribution \(\nu\) on \(\cS\times\cA\).
\end{theorem}


\begin{remark} 
    Ideally, at iteration $k,$ one would prefer an improvement guarantee on $\nu^\tr Q_{\mu_{k + }} - \nu^\tr Q_{\mu_k}.$ However, with FA, $Q_{\mu_k}$ can only be known approximately. Our result therefore guarantees improvement in terms of $f_{k + 1}$, the certified underestimator to $Q_{\mu_k}.$ These bounds are, to our knowledge, the first per-step guarantees for FA and coincide with \cite[Corollary~5]{metelli2021safe} in the tabular case where $f_{k + 1} = Q_{\mu_k}.$
\end{remark}


We end with \CRPI's convergence rate, where we show that the suboptimality gap converges at $O(1/k)$ rate to a neighborhood of $0$ and stays there  thereafter.

\begin{assumption}\label{ass:positivity}
There exists $\Delta>0$ such that $
d_{\mu_k}(s,a)\ge \Delta$ for $k \geq 0$ and $(s, a)$ with $ d_{\mu^*}(s,a)>0.$ 
\end{assumption}

\begin{theorem}
\label{prop:conv_rate}
Suppose that $T_{\mu_0}f_0 \geq f_0 \geq 0,$ and that $\|\cdot\|$ is strictly monotone. Further suppose that,  for every $(\mu, g)$ with $T_\mu g \ge g$, the set $\mathcal H(\mu,g):=\{f\in\cF:\; T_\mu f\ge f\ge g\}$ has a greatest element $f^{\max}(\mu,g)$ satisfying $Q_\mu-\epsilon\ones \le f^{\max}(\mu,g)\le Q_\mu$ for some $\epsilon>0,$ where $\ones$ is all ones vector. Then, under Assumption~\ref{ass:positivity}, if \CRPI\ is run with $\alpha_k = \frac{(1 - \gamma)^2 A_{\mu}^{\mub}(f_{k + 1})}{2 R_{\max}},$ 
$g_k:=\nu^\top(\QS-Q_{\mu_k})$ satisfies $g_k \leq \max\left\{2C_1/k, \delta\right\},$ 
where $\delta = \sqrt{2C_1 (C_2 \epsilon^2 + \epsilon)} + C_2 \epsilon^2 + \epsilon = O(\sqrt{\epsilon} + \epsilon^2 )$ with 
$C_1 = \frac{16 R_{\max} \gamma}{(1 - \gamma)^3 \Delta^2}$ and $C_2 = \frac{(1 - \gamma) \Delta^2}{4 R_{\max} \gamma}.$ Moreover, the interval $[0, \delta]$ is invariant: $g_k \leq \delta \implies g_{k + 1} \leq \delta$. 
\end{theorem}

\section{PROOF SKETCHES}
\label{s:proofs}

Here we give an outline of our proofs for Theorems \ref{thm:RPI.FA}, \ref{thm:CRPI.FA}, and \ref{prop:conv_rate}, and a complete proof of Lemma \ref{lem:Approx.Performance.Diff.Lemma}. All the other details can be found in Appendix \ref{s:apdx.Proofs}.

\begin{proof}[Sketch of Proof of Theorem~\ref{thm:RPI.FA}]
We first show in Claim~\ref{c:Q.mu.lower.Bd} that the constraint $T_\mu f \ge f$ implies $Q_\mu \ge f$. 
This follows from the monotonicity of $T_\mu$ and the fact that $(T_\mu)^m f \to Q_\mu$ as $m \to \infty$. Using this result, we prove the four statements as follows. 

\textbf{(1) Feasibility.}
Using induction, we show that the constraint set in \eqref{e:GUIDE.PE.Opt} remains feasible at every iteration. 
The initialization ensures feasibility at $k=0$. 
Assuming $f_k$ is feasible, Claim~\ref{c:Q.mu.lower.Bd} then yields $f_k \le Q_{\mu_k}$, which bounds the feasible set. 
Since the constraint set is closed and bounded, it is compact, and hence an optimizer $f_{k+1}$ exists. 
Moreover, the optimizer satisfies
\begin{equation}
\label{eq:f_k+1_feasible_sketch}
T_{\mu_k} f_{k+1} \ge f_{k+1} \ge f_k.
\end{equation}

Now, the policy $\mu_{k+1}$ is greedy with respect to $f_{k+1}$, so $T_{\mu_{k+1}} f_{k+1} = T f_{k+1}$. 
Using $T f \ge T_{\mu_k} f$, we get
\begin{equation}
\label{eq:feasibility_preservation_sketch}
T_{\mu_{k+1}} f_{k+1} = T f_{k+1} \ge T_{\mu_k} f_{k+1} \ge f_{k+1},
\end{equation}

which shows that $f_{k+1}$ satisfies the constraints in \eqref{e:GUIDE.PE.Opt} for the $k + 1$ iteration, as desired. 

\textbf{(2) Monotonicity and lower bound.}
The inequality $f_{k+1} \ge f_k$ in \eqref{eq:f_k+1_feasible_sketch} implies that $(f_k)_{k \geq 0}$ is monotonically non-decreasing. 
Separately, since $T_{\mu_k} f_k \ge f_k$ holds at every iteration as shown in \eqref{eq:feasibility_preservation_sketch}, applying Claim~\ref{c:Q.mu.lower.Bd} yields $Q_{\mu_k} \ge f_k$ for all $k,$ as desired.

\textbf{(3) Convergence and sub-optimality gap.} Since $(f_k)_{k \geq 0}$ is monotonically non-decreasing and bounded above by $Q^\star$ through $f_k \le Q_{\mu_k} \le Q^\star$, the limit $f_\infty := \lim_{k \to \infty} f_k$ exists. Using continuity of the Bellman operator and feasibility of $(f_k)$, we pass to the limit in the constraint $T_{\mu_k} f_k \ge f_k$ to obtain $T_{\mu_\infty} f_{\infty} = T f_\infty \ge f_\infty.$ We then use contraction of $T$ and $T_{\mu_\infty}$ to relate $f_\infty$ to $Q^\star$ and $Q_{\mu_\infty}$. In particular, we show that $\|f_\infty - Q^\star\|_\infty \leq \frac{\|T f_\infty - f_\infty\|_\infty}{1-\gamma}.$ Now $T f_\infty = T_{\mu_\infty} f_\infty$ and $f_\infty \leq Q_{\mu_\infty} \leq \QS$ yield
\begin{equation*}
    \|Q_{\mu_\infty} -  \QS\|_\infty 
    \leq {} \|f_\infty - \QS\|_\infty 
    \leq {}  \frac{\|T f_\infty - f_\infty\|_\infty}{1 - \gamma} = {} \frac{\|T_{\mu_\infty} f_\infty - f_\infty\|_\infty}{1 - \gamma}.
\end{equation*}

\textbf{(4) Partial Bellman optimality.}
Since the number of deterministic policies is finite, there is a policy $\mu$ that appears infinitely often in $(\mu_k)$. 
So, there is a subsequence $(k_n)$ such that $\mu_{k_n} = \mu$ for all $n$. At these indices, feasibility implies $T f_{k_n} = T_\mu f_{k_n} \ge f_{k_n}$. 
Passing to the limit using continuity of $T$ and $T_\mu$, we get $T f_\infty = T_\mu f_\infty \ge f_\infty.$ For the sake of contradiction, suppose that $T_\mu f_\infty > f_\infty$ holds strictly. Then, by continuity and the room-to-improve property of $\cF$ at $f_\infty$, there exists a function $f \in \cF$ such that $f > f_\infty$ and $T_\mu f > f$. 
Consequently, at any index $k$ with $\mu_k = \mu,$ the monotonicity of $\|\cdot\|$ would imply that the solution $f_{k + 1}$ of the optimization problem in \eqref{e:GUIDE.PE.Opt} would have satisfied $f_{k + 1} > f_\infty,$ a contradiction.
\end{proof}

We now prove Lemma~\ref{lem:Approx.Performance.Diff.Lemma}, the approximate performance-difference lemma. 

\begin{proof}[Proof of Lemma~\ref{lem:Approx.Performance.Diff.Lemma}]
    First observe that 
    \begin{align*}
        & Q_{\mup} - f \\
        & = Q_{\mup} - T_{\mu}f + T_{\mu} f - f\\
        &\overset{(a)}{=} {} T_{\mup} Q_{\mup} - T_{\mu}f  + T_{\mu}f - f \\
        %
        %
        %
        & \overset{(b)}{=}  {} \gamma P_{\mup} [Q_{\mup} - f] + \gamma a_{\mu}^{\mup}(f)  + T_{\mu}f - f \\
        &\overset{(c)}{=} {} \gamma [\bI - \gamma P_{\mup}]^{-1}  a_{\mu}^{\mup}(f)  + [\bI - \gamma P_{\mup}]^{-1} [T_{\mu}f - f], 
    \end{align*}
    where (a) follows since $T_{\mup} Q_{\mup} = Q_{\mup},$ (b) follows from the definitions of $T_{\mup},$  $T_{\mu},$ and $a_\mu^{\mup}(f)$, and (c) follows by taking the first term to the left and then multiplying $[\bI - \gamma P_{\mup}]^{-1}$ on both sides. Multiplying the last relation on both sides by $\nu^\tr$ then gives
    \begin{equation}    \label{e:Q_mup.f.exact.relation}
        \nu^\tr [Q_{\mup} - f] = \frac{\gamma}{1 - \gamma} d_{\mup}^\tr  a_{\mu}^{\mup}(f) + \nu^\tr [\bI - \gamma P_{\mup}]^{-1} [T_{\mu}f - f].
    \end{equation}
    Now, $\nu^\tr[\bI - \gamma P_{\mup}]^{-1} [T_\mu f - f] = \sum_{h = 0}^\infty \gamma^h \nu^\tr P_{\mup}^h [T_\mu f - f].$
    Also, $T_{\mu}f \geq f,$ and $P_{\mup}^h$ and $\nu$ are made up of non-negative entries; hence,  $\sum_{h = H + 1}^\infty 
        \gamma^h \nu^\tr P_{\mup}^h [T_\mu f - f] \geq 0.$ 
    The desired result now follows. 
\end{proof}

\begin{proof}[Sketch of Proof of Theorem~\ref{thm:CRPI.FA}]
We use induction to first show that \CRPI's $(f_k, \mu_k)$ pairs satisfy $T_{\mu_k} f_k \geq f_k,$ $k \geq 0.$ The $k = 0$ case holds due to initialization. Suppose $T_{\mu_k} f_k \geq f_k$ for some $k \geq 0.$ Then, the solution $f_{k + 1}$ to \eqref{e:CRPI.PE.Opt} exists and  satisfies $T_{\mu_k} f_{k + 1} \geq f_{k + 1}.$ Using arguments as in \eqref{eq:feasibility_preservation_sketch}, it then follows that $T_{\mub_k} f_{k + 1} \geq f_{k + 1}.$ Since $\mu_{k + 1} = \alpha \mub_k + (1 - \alpha)\mu_k$ for some $\alpha \in [0, 1],$ we then have $\alpha T_{\bar{\mu}_k}f_{k+1} + (1-\alpha)T_{\mu_k}f_{k+1} \geq f_{k + 1}.$ From the definition of $T_{\mu_k}$ and $T_{\mub_k}$ and using the fact that $\alpha P_{\mu_k} + (1 - \alpha) P_{\mub_k} = P_{\mu_{k + 1}},$ it then follows that $T_{\mu_{k + 1}}f_{k + 1} \geq f_{k + 1},$ as desired. 

By reasoning analogous to the proof of Theorem~\ref{thm:RPI.FA}, the conclusions of its first three statements can be shown to hold for \CRPI\ as well. The fourth statement does not carry over as directly: RPI generates deterministic policies, whereas \CRPI\ may produce stochastic ones. Nevertheless, the argument establishing $Tf_\infty \gep f_\infty$ in Theorem~\ref{thm:RPI.FA}'s proof adapts to \CRPI\ due to the compactness of the policy simplex. 

Finally, for any $k \geq 0,$ to invoke Theorem~\ref{thm:perf.bds.FA} for $\nu^\tr Q_{\mu_{k + 1}} - \nu^\tr f_{k + 1},$ we only need $T_{\mu_k} f_{k + 1} \geq f_{k + 1},$ which we have already established above. 
\end{proof}


\section{EXPERIMENTS}

We evaluate our methods against three model-based baselines: CPI, USPI, and AMPI-Q. We conduct our experiments  on two environments with linear FA: the inventory control problem \citep{bertsekas2012dynamic} (dense rewards) and the chain walk \citep{lagoudakis2003least} (sparse rewards). 
We also test a model-free deep RL instantiation of RPI using the DDPG framework \citep{ddpg} on MuJoCo continuous-control tasks \citep{brockman2016openai, todorov2012mujoco}.

For each environment we report learning curves, terminal performance, and the area under the learning curve (AUC) as a measure of sample efficiency. 
Additional results on CartPole-v1 and InvertedPendulum-v5 appear in our earlier work \citep{icc_rpi}. 
Implementation details—such as initialization (\ref{sec:init_linear_fa}), environment descriptions (\ref{sec:env_details}), and computational resources and solvers (\ref{sec:comp_res_solvers})—are provided in the appendix. All code and experimental outputs are publicly available at \url{https://github.com/EshwarSR/RPI}.

    \begin{figure*}[t]
\centering

\begin{minipage}[t]{0.65\textwidth}
\centering
\includegraphics[width=\linewidth]{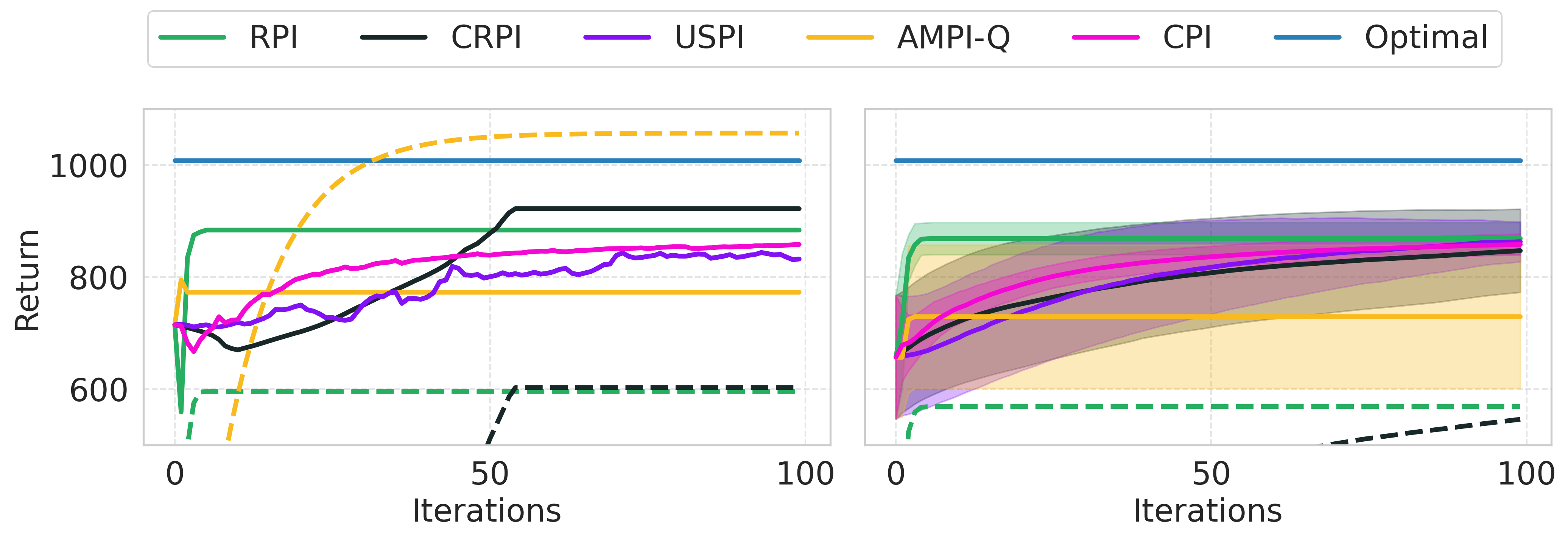}
\end{minipage}\hfill
\begin{minipage}[t]{0.33\textwidth}
\vspace*{-9.5\baselineskip}
\centering
\small
\setlength{\tabcolsep}{1pt}
\renewcommand{\arraystretch}{1.1}
\begin{tabular}{lcc}
\toprule
\multicolumn{3}{c}{\textbf{Env: Inventory Control}}\\
\midrule
\textbf{Algo.} &
\makecell{\textbf{AUC} $(\times 10^3)$} &
\makecell{\textbf{Terminal}\\\textbf{Perf. $(\times 10^2)$}}\\
\midrule
RPI & $\mathbf{85.70 \pm 2.80}$ & $\mathbf{8.69 \pm 0.29}$\\
CRPI & $78.44 \pm 8.97$ & $8.47 \pm 0.74$\\
USPI & $78.78 \pm 7.56$ & $8.63 \pm 0.36$\\
AMPI\text{-}Q & $72.08 \pm 12.47$ & $7.29 \pm 1.28$\\
CPI & $80.89 \pm 1.82$ & $8.58 \pm 0.19$\\
\bottomrule
\end{tabular}
\end{minipage}

\vspace{0.25cm}

\begin{minipage}[t]{0.65\textwidth}
\centering
\includegraphics[width=\linewidth]{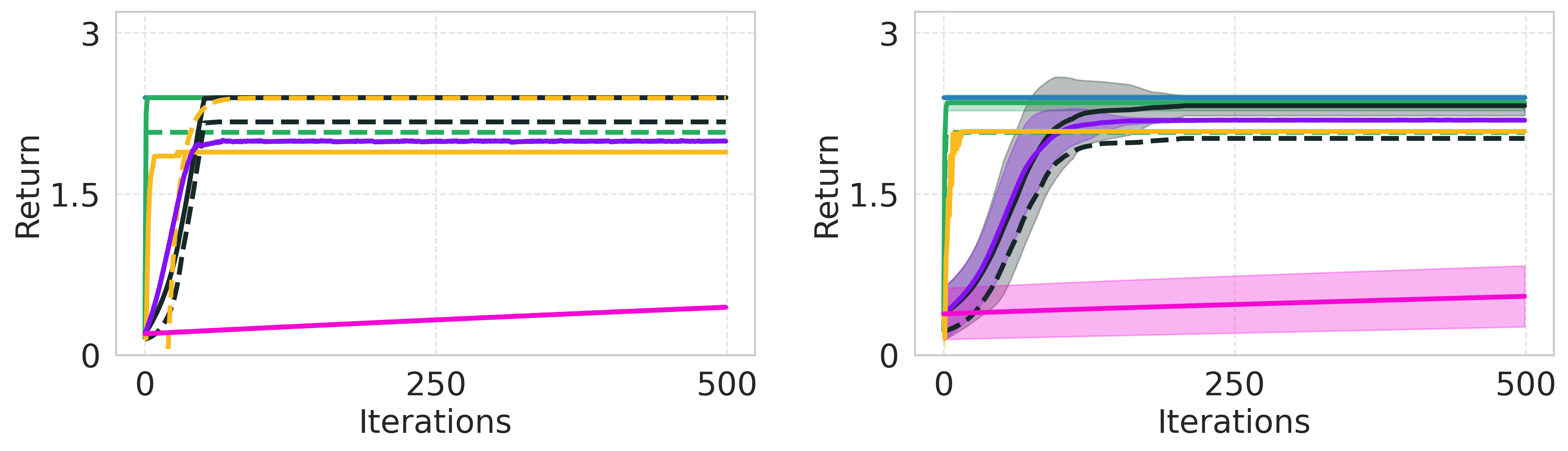}
\end{minipage}\hfill
\begin{minipage}[t]{0.33\textwidth}
\vspace*{-8\baselineskip}
\centering
\small
\setlength{\tabcolsep}{3pt}
\renewcommand{\arraystretch}{1.1}
\begin{tabular}{lcc}
\toprule
\multicolumn{3}{c}{\textbf{Env: Chain Walk}}\\
\midrule
\textbf{Algo.} &
\makecell{\textbf{AUC} $(\times 10^2)$} &
\makecell{\textbf{Terminal}\\\textbf{Perf.}}\\
\midrule
RPI & $\mathbf{11.71 \pm 0.36}$ & $\mathbf{2.35 \pm 0.07}$\\
CRPI & $10.42 \pm 0.79$ & $2.32 \pm 0.09$\\
USPI & $9.94 \pm 0.34$ & $2.19 \pm 0.01$\\
AMPI\text{-}Q & $10.32 \pm 0.01$ & $2.08 \pm 0.00$\\
CPI & $2.34 \pm 1.31$ & $0.55 \pm 0.28$\\
\bottomrule
\end{tabular}
\end{minipage}

\caption{
Inventory Control and Chain Walk with linear function approximation.
\textbf{Left:} Training curve of a single representative run (solid: true return, dashed: estimated return).
\textbf{Center:} Averaged training curves across random seeds (100 for Inventory Control, 25 for Chain Walk; solid: mean return, shaded: mean $\pm$ 1 std).
\textbf{Right:} Key metrics table reporting terminal performance and AUC (mean $\pm$ std).
\textbf{Summary:} RPI and CRPI maintains value estimates that lower bound the true returns, while those of CPI and USPI are very close to true returns. Nevertheless, 
RPI achieves the best average terminal performance and AUC across runs, indicating faster and more sample-efficient learning.
CRPI exhibits larger variance across runs in Inventory Control, which leads to trajectories that outperform RPI, as illustrated in the representative run.
AMPI-Q tends to overestimate and can exceed the optimal value, while CPI performs poorly in the sparse-reward Chain Walk setting.
}
\label{fig:rpi_crpi_combined}
\end{figure*}

\subsection{Model-Based Experiments}
\label{subsec:model_based_experiments}

Here we describe our inventory control (dense reward) and chain walk (sparse reward) experiments. 

\paragraph{Inventory Control.}
We consider the inventory-control setup of \citet{bertsekas2012dynamic} with $M=49$ (so $|\cS|=|\cA|=50$), unit cost $c=5$, holding cost $h=1$, selling price $p=10$, and discount factor $\gamma=0.9$. Demand is uniform on $\{0,\dots,M\}$. We use linear FA with feature dimension $d=75$, where the entries of $\Phi \in \mathbb{R}^{SA \times d}$ are sampled uniformly from $[1,5]$.

Figure~\ref{fig:rpi_crpi_combined} (top row) shows the results. The main highlight is that, under FA, preserving a lower-bound structure can be more useful than closely tracking true values. Although CPI and USPI produce estimates often very close to the true returns, they do not translate this into better policy performance. By contrast, RPI attains the best average terminal performance and AUC, suggesting that its conservative updates lead to faster and more reliable improvement. CRPI remains competitive but exhibits higher variance across runs.

\paragraph{Chain Walk.}
We next consider the Chain Walk domain \citep{lagoudakis2003least}, a sparse-reward MDP with $N=50$ states arranged in a linear chain and actions \texttt{Left} and \texttt{Right}. The agent moves in the intended direction with probability $0.9$ and in the opposite direction with probability $0.1$. Rewards are sparse, obtained only in two states located $N/4$ from the left and right boundaries. We use linear FA with feature dimension $d=90$, where entries are sampled uniformly from $[1,5]$, and discount factor $\gamma=0.9$. Performance in this domain is sensitive to feature choice. For a fair comparison, we sample 10 candidate feature matrices, evaluate each algorithm over 25 random seeds on each matrix, and report results corresponding to the best-performing matrix.

Figure~\ref{fig:rpi_crpi_combined} (bottom) reports results for this task. RPI achieves the highest terminal performance and AUC on average, while CRPI is competitive. 
AMPI-Q again exhibits value overestimation, and CPI performs poorly due to very slow learning.

\begin{figure*}[t]
\centering

\begin{subfigure}{0.95\textwidth}
    \centering
    \includegraphics[width=\linewidth]{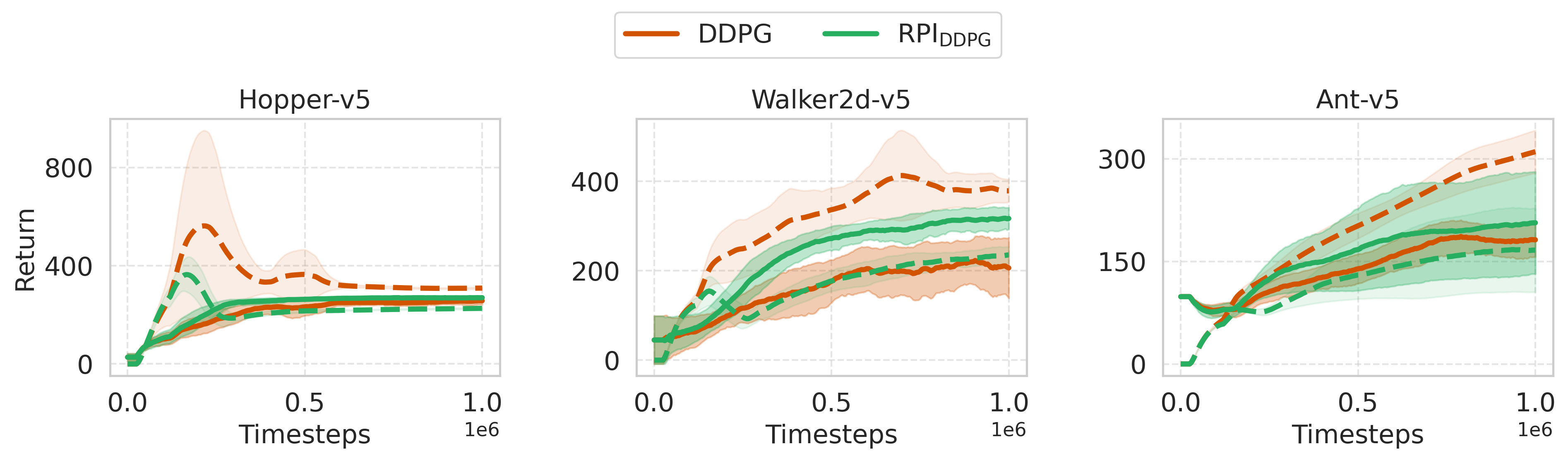}
\end{subfigure}

\vspace{0.8em}

\begin{subfigure}{0.95\textwidth}
\centering
\setlength{\tabcolsep}{5pt}

\begin{tabular}{llccccc}
\toprule
\textbf{Metric} & \textbf{Algorithm} & \textbf{Swimmer} & \textbf{Hopper} & \textbf{HalfCheetah} & \textbf{Walker2d} & \textbf{Ant} \\
\midrule

\multirow{2}{*}{\makecell[l]{Terminal\\Performance}}
& DDPG
& 19.4 $\pm$ 5.8
& 257.6 $\pm$ 9.7
& 564.8 $\pm$ 64.7
& 200.3 $\pm$ 34.1
& 182.8 $\pm$ 7.6 \\

& \rpiddpg
& \textbf{22.3 $\pm$ 1.8}
& \textbf{270.9 $\pm$ 2.0}
& \textbf{632.0 $\pm$ 50.4}
& \textbf{319.9 $\pm$ 17.1}
& \textbf{208.2 $\pm$ 73.5} \\

\midrule

\multirow{2}{*}{AUC ($\times 10^6$)}
& DDPG
& 20.1 $\pm$ 2.7
& 219.3 $\pm$ 10.3
& 483.7 $\pm$ 58.2
& 165.7 $\pm$ 4.8
& 144.4 $\pm$ 8.4 \\

& \rpiddpg
& \textbf{21.0 $\pm$ 2.6}
& \textbf{241.5 $\pm$ 6.5}
& \textbf{550.2 $\pm$ 48.2}
& \textbf{245.0 $\pm$ 10.6}
& \textbf{163.4 $\pm$ 50.5} \\

\bottomrule
\end{tabular}

\end{subfigure}

\caption{
Performance comparison between DDPG and \rpiddpg{} on MuJoCo environments.
\textbf{Top:} Learning curves for Hopper-v5, Walker2d-v5, and Ant-v5 (mean $\pm$ 1 std; solid: return, dashed: critic estimate). Remaining environments are in Appendix~\ref{sec:appendix_all_mujoco_exps}.
\textbf{Bottom:} Terminal performance and AUC (mean $\pm$ std) across all environments.
\textbf{Summary:} \rpiddpg{} maintains lower-bound value estimates, while DDPG often overestimates. \rpiddpg{} matches DDPG on simpler tasks and outperforms DDPG on harder environments.
}

\label{fig:mujoco_results}

\end{figure*}

\subsection{Model-Free Experiments}
\label{subsec:model_free_experiments}

We now demonstrate the applicability of our framework in deep RL settings, for which we develop \textsc{RPI\textsubscript{DDPG}}. This method replaces the standard Mean-Squared-Bellman-Error (MSBE) critic loss of DDPG with a custom loss function which enforces RPI's Bellman-constrained optimization.
Specifically, our inequality constraints are incorporated into the critic loss using a penalty formulation, where the penalty weight is adjusted dynamically during training. 
Implementation details are provided in Appendix~\ref{sec:deep_rpi}. 

We evaluate \rpiddpg{} against the standard DDPG algorithm on five tasks: Swimmer-v5, Hopper-v5, HalfCheetah-v5, Walker2d-v5, and Ant-v5. Figure~\ref{fig:mujoco_results} shows the learning curves for Hopper-v5, Walker2d-v5, and Ant-v5 averaged over five seeds. The plots of remaining environments are provided in Appendix~\ref{sec:appendix_all_mujoco_exps}. 

The results support the same message as in our linear-FA experiments: enforcing a conservative critic improves reliability without sacrificing performance. Across environments, the critic estimates produced by \rpiddpg{} mostly remain below the empirical returns, in line with the lower-bound structure induced by our method; when small violations occur early in training, the adaptive penalty quickly suppresses them. By contrast, DDPG often exhibits overestimation, particularly in Walker2d-v5 and Ant-v5.

Overall, \rpiddpg{} performs competitively with or better than DDPG across all tasks. 
In simpler environments such as Swimmer-v5 and Hopper-v5 both methods achieve similar performance, while in the remaining environments \rpiddpg{} attains higher returns and learns faster. 
The summary table confirms this trend: \rpiddpg{} achieves higher terminal performance and larger AUC across all environments.

\section{CONCLUSION}

We introduced RPI and its conservative variant, CRPI, to address the breakdown of classical PI guarantees under FA. The core idea is a Bellman-constrained policy-evaluation step that produces certified lower bounds and restores monotonic value estimates. We further derive a new performance-difference lower bound that allows CRPI to carry out conservative policy updates with per-step improvement guarantees.  
Empirically, these ideas translate into strong performance. RPI performs best overall, achieving the highest average terminal performance and AUC across Inventory Control and Chain Walk. CRPI remains competitive: it closely tracks RPI in Chain Walk, while in Inventory Control its conservative updates lead to higher variance but can also outperform RPI in some runs. Overall, our work offers a principled bridge between tabular-style reliability guarantees and practical RL with FA. Our guarantees, however, impose a constraint: when the function class admits no feasible improving direction, the method may stagnate. However, this behavior is not a limitation but a consequence of enforcing reliability—it prevents harmful updates. We illustrate this phenomenon in Appendix~\ref{sec:limitations}.

\subsubsection*{Acknowledgements}
 
We thank anonymous reviewers for their valuable comments to improve our work. We also thank Aniruddha Mukherjee for helping run some of the simulations. Eshwar's research is supported in part by grant from the National Payments Corporation of India (NPCI) to Indian Institute of Science (IISc) and Prime Minister's Research Fellowship. Gugan Thoppe's research is supported in part by grants from the Walmart Centre for Tech Excellence at IISc, the Indo-French Centre for the Promotion of Advanced Research Grant (CEFIPRA; Project 7102-1), Kotak IISc FinTech Grant, DST-SERB's Core Research Grant (CRG/2021/008330), and the Pratiksha Trust Young Investigator Award. Aditya Gopalan's research is supported in part by the Centre for Networked Intelligence (a Cisco Corporate Social Responsibility (CSR) Initiative), IISc. Both Gugan and Aditya are also supported by the ANRF ARG grant (ANRF/ARG/2025/011952/ENS).

\bibliography{refs}

@inproceedings{scherrer2010should,
  title={Should one compute the temporal difference fix point or minimize the Bellman Residual? The unified oblique projection view},
  author={Scherrer, Bruno},
  booktitle={Proceedings of the 27th International Conference on International Conference on Machine Learning},
  pages={959--966},
  year={2010}
}

@misc{aethelios2025beyond,
  author       = {Aethelios},
  title        = {Beyond Hype: The Brutal Truth About Deep Reinforcement Learning},
  howpublished = {\emph{Medium}},
  year         = {2025},
  month        = jun,
  day          = {5},
  url          = {https://medium.com/@Aethelios/beyond-hype-the-brutal-truth-about-deep-reinforcement-learning-a9b408ffaf4a},
  urldate      = {2025-09-03},
  note         = {Blog post}
}

@inproceedings{henderson2018deep,
  title={Deep reinforcement learning that matters},
  author={Henderson, Peter and Islam, Riashat and Bachman, Philip and Pineau, Joelle and Precup, Doina and Meger, David},
  booktitle={Proceedings of the AAAI conference on artificial intelligence},
  volume={32},
  year={2018}
}

@inproceedings{lazaric2010analysis,
  title={Analysis of a classification-based policy iteration algorithm},
  author={Lazaric, Alessandro and Ghavamzadeh, Mohammad and Munos, R{\'e}mi},
  booktitle={ICML-27th International Conference on Machine Learning},
  pages={607--614},
  year={2010},
  organization={Omnipress}
}

@inproceedings{haarnoja2018soft,
  title={Soft actor-critic: Off-policy maximum entropy deep reinforcement learning with a stochastic actor},
  author={Haarnoja, Tuomas and Zhou, Aurick and Abbeel, Pieter and Levine, Sergey},
  booktitle={International conference on machine learning},
  pages={1861--1870},
  year={2018},
  organization={Pmlr}
}

@article{hasselt2010double,
  title={Double Q-learning},
  author={Hasselt, Hado},
  journal={Advances in neural information processing systems},
  volume={23},
  year={2010}
}

@inproceedings{efroni2018beyond,
  title={Beyond the one-step greedy approach in reinforcement learning},
  author={Efroni, Yonathan and Dalal, Gal and Scherrer, Bruno and Mannor, Shie},
  booktitle={International Conference on Machine Learning},
  pages={1387--1396},
  year={2018},
  organization={PMLR}
}

@article{tsitsiklis1996analysis,
  title={Analysis of temporal-diffference learning with function approximation},
  author={Tsitsiklis, John and Van Roy, Benjamin},
  journal={Advances in neural information processing systems},
  volume={9},
  year={1996}
}

@article{samuel1959some,
  title={Some studies in machine learning using the game of checkers},
  author={Samuel, Arthur L},
  journal={IBM Journal of research and development},
  volume={3},
  number={3},
  pages={210--229},
  year={1959},
  publisher={IBM}
}

@article{samuel1967some,
  title={Some studies in machine learning using the game of checkers. II—Recent progress},
  author={Samuel, Arthur L},
  journal={IBM Journal of research and development},
  volume={11},
  number={6},
  pages={601--617},
  year={1967},
  publisher={IBM}
}

@article{ernst2005tree,
  title={Tree-based batch mode reinforcement learning},
  author={Ernst, Damien and Geurts, Pierre and Wehenkel, Louis},
  journal={Journal of Machine Learning Research},
  volume={6},
  year={2005},
  publisher={Microtome Publishing, Brookline, United States-Massachusetts}
}

@article{de2003linear,
  title={The linear programming approach to approximate dynamic programming},
  author={De Farias, Daniela Pucci and Van Roy, Benjamin},
  journal={Operations research},
  volume={51},
  number={6},
  pages={850--865},
  year={2003},
  publisher={INFORMS}
}

@article{efroni2018multiple,
  title={Multiple-step greedy policies in approximate and online reinforcement learning},
  author={Efroni, Yonathan and Dalal, Gal and Scherrer, Bruno and Mannor, Shie},
  journal={Advances in neural information processing systems},
  volume={31},
  year={2018}
}

@inproceedings{efroni2019combine,
  title={How to combine tree-search methods in reinforcement learning},
  author={Efroni, Yonathan and Dalal, Gal and Scherrer, Bruno and Mannor, Shie},
  booktitle={Proceedings of the AAAI Conference on Artificial Intelligence},
  volume={33},
  pages={3494--3501},
  year={2019}
}

@article{winnicki2021role,
  title={The role of lookahead and approximate policy evaluation in policy iteration with linear value function approximation},
  author={Winnicki, Anna and Lubars, Joseph and Livesay, Michael and Srikant, R},
  journal={CoRR},
  year={2021}
}

@article{metelli2021safe,
  title={Safe policy iteration: A monotonically improving approximate policy iteration approach},
  author={Metelli, Alberto Maria and Pirotta, Matteo and Calandriello, Daniele and Restelli, Marcello},
  journal={Journal of Machine Learning Research},
  volume={22},
  number={97},
  pages={1--83},
  year={2021}
}

@article{lagoudakis2003least,
  title={Least-squares policy iteration},
  author={Lagoudakis, Michail G and Parr, Ronald},
  journal={Journal of machine learning research},
  volume={4},
  number={Dec},
  pages={1107--1149},
  year={2003}
}

@article{bertsekas2011approximate,
  title={Approximate policy iteration: A survey and some new methods},
  author={Bertsekas, Dimitri P},
  journal={Journal of Control Theory and Applications},
  volume={9},
  number={3},
  pages={310--335},
  year={2011},
  publisher={Springer}
}

@article{young2020understanding,
  title={Understanding the pathologies of approximate policy evaluation when combined with greedification in reinforcement learning},
  author={Young, Kenny and Sutton, Richard S},
  journal={arXiv preprint arXiv:2010.15268},
  year={2020}
}

@article{patterson2024empirical,
  title={Empirical design in reinforcement learning},
  author={Patterson, Andrew and Neumann, Samuel and White, Martha and White, Adam},
  journal={Journal of Machine Learning Research},
  volume={25},
  number={318},
  pages={1--63},
  year={2024}
}

@article{gopalan2025does,
  title={Does {DQN} learn?},
  author={Gopalan, Aditya and Thoppe, Gugan},
  journal={IEEE Transaction on Automatic Control (to appear)},
  year={2025}
}

@article{howard1960dynamic,
  title={Dynamic programming and markov processes.},
  author={Howard, Ronald A},
  year={1960},
  publisher={John Wiley}
}

@inproceedings{thrun1993issues,
  title={Issues in using function approximation for reinforcement learning},
  author={Thrun, Sebastian and Schwartz, Anton},
  booktitle={Proceedings of the 1993 connectionist models summer school},
  pages={255--263},
  year={1993}
}

@article{schulman2017proximal,
  title={Proximal policy optimization algorithms},
  author={Schulman, John and Wolski, Filip and Dhariwal, Prafulla and Radford, Alec and Klimov, Oleg},
  journal={arXiv preprint arXiv:1707.06347},
  year={2017}
}

@inproceedings{van2016deep,
  title={Deep reinforcement learning with double q-learning},
  author={Van Hasselt, Hado and Guez, Arthur and Silver, David},
  booktitle={Proceedings of the AAAI conference on artificial intelligence},
  volume={30},
  year={2016}
}

@inproceedings{baird1995residual,
  title={Residual algorithms: Reinforcement learning with function approximation},
  author={Baird, Leemon},
  booktitle={Proceedings of the twelfth international conference on machine learning},
  pages={30--37},
  year={1995}
}

@inproceedings{kakade2002approximately,
  title={Approximately optimal approximate reinforcement learning},
  author={Kakade, Sham and Langford, John},
  booktitle={Proceedings of the nineteenth international conference on machine learning},
  pages={267--274},
  year={2002}
}

@book{bertsekas1996neuro,
  title={Neuro-dynamic programming},
  author={Bertsekas, Dimitri and Tsitsiklis, John N},
  year={1996},
  publisher={Athena Scientific}
}

@article{cvxpy1,
  author  = {Steven Diamond and Stephen Boyd},
  title   = {{CVXPY}: {A} {P}ython-embedded modeling language for convex optimization},
  journal = {Journal of Machine Learning Research},
  year    = {2016},
  volume  = {17},
  number  = {83},
  pages   = {1--5},
}

@article{cvxpy2,
  author  = {Agrawal, Akshay and Verschueren, Robin and Diamond, Steven and Boyd, Stephen},
  title   = {A rewriting system for convex optimization problems},
  journal = {Journal of Control and Decision},
  year    = {2018},
  volume  = {5},
  number  = {1},
  pages   = {42--60},
}

@article{ampi,
  author  = {Bruno Scherrer and Mohammad Ghavamzadeh and Victor Gabillon and Boris Lesner and Matthieu Geist},
  title   = {Approximate Modified Policy Iteration and its Application to the Game of Tetris},
  journal = {Journal of Machine Learning Research},
  year    = {2015},
  volume  = {16},
  number  = {49},
  pages   = {1629--1676},
  url     = {http://jmlr.org/papers/v16/scherrer15a.html}
}

@inproceedings{trpo,
  title={Trust region policy optimization},
  author={Schulman, John and Levine, Sergey and Abbeel, Pieter and Jordan, Michael and Moritz, Philipp},
  booktitle={International conference on machine learning},
  pages={1889--1897},
  year={2015},
  organization={PMLR}
}

@inproceedings{td3,
  title={Addressing function approximation error in actor-critic methods},
  author={Fujimoto, Scott and Hoof, Herke and Meger, David},
  booktitle={International conference on machine learning},
  pages={1587--1596},
  year={2018},
  organization={PMLR}
}

@article{ddpg,
  title={Continuous control with deep reinforcement learning},
  author={Lillicrap, Timothy P and Hunt, Jonathan J and Pritzel, Alexander and Heess, Nicolas and Erez, Tom and Tassa, Yuval and Silver, David and Wierstra, Daan},
  journal={arXiv preprint arXiv:1509.02971},
  year={2015}
}

@article{dqn,
  title={Human-level control through deep reinforcement learning},
  author={Mnih, Volodymyr and Kavukcuoglu, Koray and Silver, David and Rusu, Andrei A and Veness, Joel and Bellemare, Marc G and Graves, Alex and Riedmiller, Martin and Fidjeland, Andreas K and Ostrovski, Georg and others},
  journal={nature},
  volume={518},
  number={7540},
  pages={529--533},
  year={2015},
  publisher={Nature Publishing Group}
}

@book{bertsekas2012dynamic,
  title={Dynamic programming and optimal control},
  author={Bertsekas, Dimitri},
  volume={1},
  year={2012},
  publisher={Athena scientific}
}

@article{brockman2016openai,
  title={Openai gym},
  author={Brockman, Greg and Cheung, Vicki and Pettersson, Ludwig and Schneider, Jonas and Schulman, John and Tang, Jie and Zaremba, Wojciech},
  journal={arXiv preprint arXiv:1606.01540},
  year={2016}
}

@inproceedings{todorov2012mujoco,
  title={MuJoCo: A physics engine for model-based control},
  author={Todorov, Emanuel and Erez, Tom and Tassa, Yuval},
  booktitle={2012 IEEE/RSJ International Conference on Intelligent Robots and Systems},
  pages={5026--5033},
  year={2012},
  organization={IEEE},
  doi={10.1109/IROS.2012.6386109}
}

@article{haviv1984,
  title={Perturbation bounds for the stationary probabilities of a finite Markov chain},
  author={Haviv, Moshe and Van der Heyden, Ludo},
  journal={Advances in Applied Probability},
  volume={16},
  number={4},
  pages={804--818},
  year={1984},
  publisher={Cambridge University Press}
}

@misc{gurobi,
  author = {{Gurobi Optimization, LLC}},
  title = {{Gurobi Optimizer Reference Manual}},
  year = 2024,
  url = "https://www.gurobi.com"
}

@INPROCEEDINGS{icc_rpi,
  booktitle={2025 Eleventh Indian Control Conference (ICC-11)}, 
  title={Reliable Policy Iteration: Performance Robustness Across Architecture and Environment Perturbations},
  author={Eshwar, SR and Mukherjee, Aniruddha and Saha, Kintan and Agarwal, Krishna and Thoppe, Gugan and Gopalan, Aditya and Dalal, Gal},
  year={2025},
  volume={},
  number={},
  pages={},
  doi={}}

@article{liu2019neural,
  title={Neural trust region/proximal policy optimization attains globally optimal policy},
  author={Liu, Boyi and Cai, Qi and Yang, Zhuoran and Wang, Zhaoran},
  journal={Advances in neural information processing systems},
  volume={32},
  year={2019}
}

@article{agarwal2021theory,
  title={On the theory of policy gradient methods: Optimality, approximation, and distribution shift},
  author={Agarwal, Alekh and Kakade, Sham M and Lee, Jason D and Mahajan, Gaurav},
  journal={Journal of Machine Learning Research},
  volume={22},
  number={98},
  pages={1--76},
  year={2021}
}

@article{alfano2023novel,
  title={A novel framework for policy mirror descent with general parameterization and linear convergence},
  author={Alfano, Carlo and Yuan, Rui and Rebeschini, Patrick},
  journal={Advances in Neural Information Processing Systems},
  volume={36},
  pages={30681--30725},
  year={2023}
}

@inproceedings{ansel2024pytorch,
  title={Pytorch 2: Faster machine learning through dynamic python bytecode transformation and graph compilation},
  author={Ansel, Jason and Yang, Edward and He, Horace and Gimelshein, Natalia and Jain, Animesh and Voznesensky, Michael and Bao, Bin and Bell, Peter and Berard, David and Burovski, Evgeni and others},
  booktitle={Proceedings of the 29th ACM international conference on architectural support for programming languages and operating systems, volume 2},
  pages={929--947},
  year={2024}
}

\clearpage
\appendix
\thispagestyle{empty}

\onecolumn
\aistatstitle{Appendix}

\section{PROOFS}
\label{s:apdx.Proofs}

We begin with Theorem~\ref{thm:RPI.FA}'s proof.  First, we derive a key relation from the constraints in \eqref{e:GUIDE.PE.Opt}. 

\begin{claim}
\label{c:Q.mu.lower.Bd}
    For a policy $\mu$ and a vector $f \in \bR^{SA},$ the condition $T_\mu f \geq f$ implies $Q_\mu \geq f.$
\end{claim}
\begin{proof}
The given condition and the monotonicity of $T_\mu$ imply $(T_\mu)^m f \geq \cdots \geq T_\mu f \geq f$
for any $m \geq 0.$ Hence, $Q_\mu = \lim_{m \to \infty} (T_{\mu} )^m f \geq f,$ as desired.
\end{proof}

\begin{proof}[Proof of Theorem~\ref{thm:RPI.FA}]
We now use Claim~\ref{c:Q.mu.lower.Bd} and induction to prove the first statement. The condition  $T_{\mu_0} f_0 \geq f_0$ ensures that $f_0$ is feasible with respect to the constraints in \eqref{e:GUIDE.PE.Opt} for $k = 0$. Now, suppose $f = f_k$ satisfies the constraints in \eqref{e:GUIDE.PE.Opt} for some arbitrary $k \geq 0.$ Then, Claim~\ref{c:Q.mu.lower.Bd} shows that
\begin{equation}
\label{e:constraint.implies.Qmu.upper.Bd}
\begin{aligned}
    \{f \in \cF: & T_{\mu_k} f \geq f \geq f_k\} \\
    = {} & \cF \cap \{f \in \bR^{SA}: T_{\mu_k} f \geq f \geq f_k\} \\
    \subseteq {} & \cF \cap \{f \in \bR^{SA}: Q_{\mu_k} \geq f \geq f_k\}.
\end{aligned}
\end{equation}
Since $\cF$ and $\{f \in \bR^{SA}: T_{\mu_k} f \geq f \geq f_k\}$ are closed in $\bR^{SA}$ and $\{f \in \bR^{SA}: Q_{\mu_k} \geq f \geq f_k\}$ is bounded, the constraint set $\{f \in \cF: T_{\mu_k} f \geq f \geq f_k\}$ itself is closed and bounded and, hence, compact. Because the objective function is continuous,  compactness implies that a solution $f_{k + 1}$ to \eqref{e:GUIDE.PE.Opt} exists and it satisfies
\begin{equation}
\label{e:Phi.u_k.monotonicity}
    T_{\mu_k} f_{k + 1} \geq f_{k + 1} \geq f_k.
\end{equation}
The new policy $\mu_{k + 1}$ obtained subsequently from $f_{k + 1}$ then satisfies
\begin{equation}
\label{e:RPI.induction.constraint.feasibility}
    T_{\mu_{k + 1}} f_{k + 1} = T f_{k + 1} \geq T_{\mu_k} f_{k + 1} \geq f_{k + 1},
\end{equation}
which shows that $f_{k + 1}$ satisfies the constraints in \eqref{e:GUIDE.PE.Opt} for the $k + 1$ iteration, as desired. 

Now consider the second statement. From the rightmost inequality in \eqref{e:Phi.u_k.monotonicity}, it follows that $(f_k)_{k \geq 0}$ is non-decreasing. On the other hand, for any $k \geq 0,$ our first statement shows that $f_k$ satisfies the constraints in \eqref{e:GUIDE.PE.Opt}; Claim~\ref{c:Q.mu.lower.Bd} then shows that $Q_{\mu_k} \geq f_k.$ Therefore, $f_k$ is a lower bound on $Q_{\mu_k}$ for any $k \geq 0,$ and $(f_k)_{k \geq 0}$ is monotonically non-decreasing, as desired. 

With regard to the third statement, since $(f_k)_{k \geq 0}$ is monotonically non-decreasing and $f_k \leq Q_{\mu_k} \leq \QS,$ it follows that $f_{\infty} := \lim_{k \to \infty} f_k$ exists. Separately, for any $k \geq 0,$ since $f_{k + 1} \geq f_k$ and since $T_{\mu_k}$ is monotone, we have $T_{\mu_k} f_{k + 1} \geq T_{\mu_k} f_k.$ This, when combined with \eqref{e:RPI.induction.constraint.feasibility}, then shows that $(T_{\mu_k} f_k)_{k \geq 0}$ is monotone. Therefore, 
\begin{equation}
\label{e:T.infty.mu.infty.f.infty.rel}
    T_{\mu_\infty} f_{\infty} \overset{(a)}{=} T f_{\infty} \overset{(b)}{=} T(\lim_{k \to \infty} f_k) \overset{(c)}{=} \lim_{k \to \infty} T f_k  \overset{(d)}{\geq } \lim_{k \to \infty} T_{\mu_k} f_k \overset{(e)}{\geq} \lim_{k \to \infty} f_k \overset{(f)}{=} f_{\infty},
\end{equation}
where (a) holds because $\mu_\infty$ is greedy with respect $f_{\infty},$ (b) and (f) hold from the definition of $f_{\infty},$ (c) holds since $T$ is continuous,
(d) holds since $Tf_k  \geq T_{\mu_k} f_k$ and $(T_{\mu_k} f_k)_{k \geq 0}$ is monotone and, hence, convergent, while (e) holds since $f_k$ satisfies the constraints in \eqref{e:GUIDE.PE.Opt} as shown in our first statement. 

Next, since $T\QS = \QS,$ observe that
\begin{multline*}
    \|f_\infty - \QS\|_\infty \leq \|T f_\infty - f_\infty\|_\infty + \|T f_\infty - T \QS\|_\infty \leq \|T f_\infty - f_\infty\|_\infty + \gamma \|f_\infty -  \QS\|_\infty,
\end{multline*}
where the last inequality follows since $T$ is a contraction. Hence,
\begin{equation}
\label{e:Phi.L.QS.Bd}
     \|f_\infty - \QS\|_\infty \leq \frac{\|T f_\infty - f_\infty\|_\infty}{1 - \gamma}.
\end{equation}
Separately, from \eqref{e:T.infty.mu.infty.f.infty.rel}, we have $T_{\mu_\infty} f_\infty \geq f_\infty.$ Hence, 
\begin{equation}
\label{e:f.infty.Q*.Rel}
    f_\infty \leq Q_{\mu_\infty} \leq \QS,
\end{equation}  
where the first inequality follows from Claim~\ref{c:Q.mu.lower.Bd}. Therefore, 
\[
    \|\QS - Q_{\mu_\infty}\|_\infty \overset{(a)}{\leq} \|\QS - f_\infty \|_\infty \overset{(b)}{\leq} \frac{\|T f_\infty - f_\infty\|_\infty}{1 - \gamma} \overset{(c)}{=} \frac{\|T_{\mu_\infty} f_\infty - f_\infty\|_\infty}{1 - \gamma}
\]
where (a) follows from \eqref{e:f.infty.Q*.Rel}, (b) from \eqref{e:Phi.L.QS.Bd}, while (c) since $T f_\infty = T_{\mu_\infty} f_\infty.$

Finally, we discuss the fourth statement on the partial satisfiability of the Bellman equations by $f_\infty.$ Since the number of state-action pairs is finite, we only have finitely many deterministic policies. Hence, among $(\mu_k)_{k \geq 0},$ there exists a deterministic policy (say $\mu$) that repeats infinitely often. That is, there is a subsequence $(k_n)_{n \geq 0}$ such that $\mu_{k_n} = \mu$ for all $n \geq 0.$ Now, since $\lim_{n \to \infty} f_{k_n} = f_{\infty}$ and, at iteration $k_n,$ $T f_{k_n} = T_{\mu} f_{k_n} = T_{\mu_{k_n}} f_{k_n} \geq f_{k_n},$ the continuity of $T$ and $T_\mu$ implies $T f_\infty = T_{\mu} f_\infty \geq f_\infty.$ 
Suppose $T f_\infty = T_{\mu} f_\infty > f_\infty,$ i.e., the strict inequality holds on all coordinates. Let $\eta := \min_{s, a}|T_\mu f_\infty(s, a) - f_\infty(s, a)|,$ where $T_{\mu} f_\infty(s, a)$ and $f_\infty(s,a)$ denote the $(s, a)$-th coordinate of $T_{\mu} f_\infty$ and $f_\infty,$ respectively. Then, from the continuity of $T_\mu,$ it follows that there exist some $\delta$ and $\epsilon$ such that $0 < \delta, \epsilon <  \eta/2$ and, for any $f \in \bR^{SA}$ satisfying $\|f - f_\infty\|_\infty \leq \delta,$ we have $\|T_\mu f - T_{\mu} f_\infty\|_\infty \leq \epsilon.$ Now, the given condition that $\cF$ has room to improve at $f_\infty$ implies there exists a $f \in \cF$ such that $f > f_\infty$ and $\|f - f_\infty\|_\infty \leq \delta.$ Hence, for this $f,$ we have $T_{\mu} f > f.$ Consequently, at any index $k$ with $\mu_k = \mu,$ the monotonicity of $\|\cdot\|$ would imply that the solution $f_{k + 1}$ of the optimization problem in \eqref{e:GUIDE.PE.Opt} would have satisfied $f_{k + 1} > f_\infty,$ a contradiction. This shows that $T f_\infty \gep f_\infty,$ as desired.
\end{proof}

We next derive Propositions~\ref{prop:RPI.PI.generalization} and \ref{prop:RPI.l1.projection}.    

\begin{proof}[Proof of Proposition~\ref{prop:RPI.PI.generalization}]
Since $\cF = \bR^{SA},$ we have that $Q_{\mu} \in \cF$ for any $\mu.$ Now, for iteration $k \geq 0$ of RPI, Claim~\ref{c:Q.mu.lower.Bd} implies that every element in $\{f \in \cF: T_{\mu_k}f \geq f \geq f_k\}$ satisfies $Q_{\mu_k} \geq f \geq f_k$ or, equivalently, $Q_{\mu_k} - f_k \geq f - f_k \geq 0.$ The strict monotonicity of $\|\cdot\|$ then implies that  $\|Q_{\mu_k} - f_k\| \geq \|f - f_k\|.$ Hence, $f_{k + 1} = Q_{\mu_k}.$ Thus, RPI mirrors vanilla PI. Therefore, by invoking the classical convergence proof for PI \citep{bertsekas1996neuro}, we now get $f_\infty = \QS$ and $\mu_\infty$ is the optimal policy, as desired. 
\end{proof}

\begin{proof}[Proof of Proposition~\ref{prop:RPI.l1.projection}]
For any $f$ satisfying the constraints in \eqref{e:GUIDE.PE.Opt}, we have from Claim~\ref{c:Q.mu.lower.Bd} that  $Q_{\mu_k} \geq f \geq f_k.$ Hence, for any weight vector $w \in \bR^{SA}$ with strictly positive values, we have
\begin{align*}
    \|f - f_k\|_{w, 1} = {} & \sum_{s, a} w(s, a) [f(s, a) - f_k(s, a)] \\
    = {} & \sum_{s, a} w(s, a) [f(s, a) - Q_{\mu_k}(s, a)  + Q_{\mu_k}(s, a) - f_k(s, a)]\\
    = {} & -\|Q_{\mu_k} - f\|_{w, 1} + \|f_k - Q_{\mu_k}\|_{w, 1}.
\end{align*}
Since the rightmost term in the last expression is independent of $f,$ the desired result follows. 
\end{proof}

Next we prove Proposition~\ref{prop:perf.Diff.Quad.Bd}. To this end, we need the following technical result that mirrors Lemma~1 of \citep{metelli2021safe}. 

\begin{lemma}
 \label{lem:Discounted.State.Distrn.Diff.Lemma}
 For any initial distribution $\nu$ and stationary policies $\mup$, $\mu$, we have
 \begin{equation}
 \label{e:disc.state.diff.lemma}
 \|d_{\mup} - d_{\mu}\|_1 \leq \frac{\gamma}{1 - \gamma} \|\mup - \mu\|_{1, \delta_\mu}. 
 \end{equation}
 \end{lemma}

\begin{proof}
    We have
    \begin{align*}
        d_{\mup}^\tr - d_{\mu}^\tr = {} & (1 - \gamma) \nu^\tr[\bI - \gamma P_{\mup}]^{-1} -  (1 -  \gamma) \nu^\tr [\bI - \gamma P_\mu]^{-1} \\
        = {} & (1 - \gamma) \nu^\tr [\bI + \gamma  (\bI - \gamma P_{\mup})^{-1} P_\mup] -  (1 -  \gamma) \nu^\tr [\bI + \gamma (\bI- \gamma P_\mu)^{-1} P_\mu ]\\
        = {} & \gamma [d_{\mup}^\tr P_{\mup} - d_{\mu}^\tr P_{\mu}] \\
        = {} & \gamma [d_{\mup}^\tr - d_{\mu}^\tr ]P_{\mup} + \gamma d_{\mu}^\tr [P_{\mup} - P_{\mu}] \\
        = {} & \gamma d_{\mu}^\tr [P_{\mup} - P_{\mu}] [\bI - \gamma P_{\mup}]^{-1}.
    \end{align*}
    
    Hence, it follows that 
    \begin{align}
        \|d_{\mup} - d_{\mu}\|_1 = {} & \|d_{\mup}^\tr - d_{\mu}^\tr\|_\infty \nonumber \\
        \leq  {} & \gamma \|d_{\mu}^\tr [P_\mup - P_{\mu}]\|_\infty\ \|[\bI - \gamma P_{\mu}]^{-1}\|_\infty \nonumber \\
        = {} & \frac{\gamma}{1 - \gamma } \|d_{\mu}^\tr [P_\mup - P_{\mu}]\|_\infty,
        \label{e:dmu.minus.dmup.Bd}
    \end{align}
    where the last inequality follows from the fact that $\|(\bI - \gamma P_\mu)^{-1}\|_\infty \leq \sum_{k = 0}^\infty \gamma^k \|(P_\mu)^k\|_\infty = 1/(1 - \gamma).$

    Now, 
    \begin{align}
        \|d_{\mu}^\tr [P_\mup - P_\mu]\|_\infty = {} & \sum_{s', a'} \Big|\sum_{s, a} d_\mu(s, a)   P(s'|s, a) [\mup(a'|s') - \mu(a'|s')] \Big| \nonumber \\
        \leq {} & \sum_{s', a'} \sum_{s, a} d_{\mu}(s, a) P(s'|s, a) \Big|\mup(a'|s') - \mu(a'|s') \Big| \nonumber \\
        = {} & \sum_{s, a, s'} d_{\mu}(s, a) P(s'|s, a) \sum_{a'}  \Big|\mup(a'|s') - \mu(a'|s') \Big| \nonumber \\
        \leq {} & \sum_{s, a, s'} d_{\mu}(s, a) P(s'|s, a)  \|\mup(\cdot|s') - \mu(\cdot|s')\|_1 \nonumber\\
        = {} & \|\mup - \mu\|_{1, \delta_\mu} \label{e:mup.minus.mu.Bd}.
    \end{align}
    where $\delta_\mu$ is defined as in Section~\ref{subsec:CRPI}. The desired result now follows from \eqref{e:dmu.minus.dmup.Bd} and \eqref{e:mup.minus.mu.Bd}.
\end{proof}

\begin{proof}[Proof of Proposition~\ref{prop:perf.Diff.Quad.Bd}]
    Starting from Lemma~\ref{lem:Approx.Performance.Diff.Lemma}, we get 
    \begin{equation}
    \label{e:nu.Q.mup.nu.f.Bd1}
        \frac{1-\gamma}{\gamma} (\nu^\tr Q_{\mup}  - \nu^\tr f) 
            \ge d_{\mup}^\tr\,a_\mu^{\mup}(f)  + \frac{1 - \gamma}{\gamma} \sum_{h = 0}^H \gamma^h \nu^\tr P_{\mup}^h [T_{\mu}f - f].
    \end{equation}
    Now, for any $\mu_b,$ observe that 
    \begin{align}
            d_{\mup}^\tr a_{\mu}^{\mup}(f) 
            & =  d_{\mu_b}^\tr\,a_\mu^{\mup}(f) + (d_{\mup}^\tr - d_{\mu_b}^\tr) a_\mu^{\mup}(f) \nonumber \\
            &\overset{(a)}{\geq} d_{\mu_b}^\tr a_\mu^{\mup}(f) - \bigl|(d_{\mup}-d_{\mu_b})^\tr a_\mu^{\mup}(f)\bigr| \nonumber \\
            &\overset{(b)}{\geq} d_{\mu_b}^\tr\,a_\mu^{\mup}(f) - \|(d_{\mup}-d_{\mu_b})^\tr \|_1\frac{\mathrm{sp}(a_\mu^{\mup}(f))}{2} \nonumber \\
            &\overset{(c)}{\geq} d_{\mu_b}^\tr a_\mu^{\mup}(f)
            -  \frac{\gamma}{1 - \gamma} \|\mup - \mu_b\|_{1, \delta_{\mu_b}}\frac{\mathrm{sp}(a_\mu^{\mup}(f))}{2}, \label{e:nu.Q.mup.nu.f.Bd2}
    \end{align}
where $(a)$ follows from the fact that $x + y \geq x -|y|$ for any real numbers $x$ and $y,$ $(b)$ follows from \cite[Corollary 2.4]{haviv1984}, and $(c)$ follows from Lemma~\ref{lem:Discounted.State.Distrn.Diff.Lemma}.

Separately, for $\mu_b = \mu$ and $\mup = \alpha \mub + (1 - \alpha) \mu,$ observe that $P_{\mup} = \alpha P_{\mub} + (1 - \alpha) P_{\mu}$ and, hence, 
    \begin{equation}
    \label{e:nu.Q.mup.nu.f.Bd3}
        \begin{aligned}
           d_{\mu_b}^\tr a_{\mu}^{\mup}(f)
           = {} & \alpha d_{\mu}^\tr [P_{\mub}- P_{\mu}]f
           = \alpha A_{\mu}^{\mub}(f) \\
           \|\mup(\cdot|s) - \mu(\cdot|s)\|_1 = {} &  \alpha \|\mub(\cdot|s) - \mu(\cdot|s)\|_1\\
           \sp(a_\mu^{\mup}(f)) = {} & \alpha\  \sp(a_\mu^{\mub}(f)).
        \end{aligned}
    \end{equation}

    Hence, by combining \eqref{e:nu.Q.mup.nu.f.Bd1}, \eqref{e:nu.Q.mup.nu.f.Bd2}, and \eqref{e:nu.Q.mup.nu.f.Bd3}, it follows that
    \begin{equation}
     \frac{1-\gamma}{\gamma} (\nu^\tr Q_{\mup}  - \nu^\tr f) 
    \ge \alpha A_{\mu}^{\mub}(f) - \frac{\gamma}{2(1 - \gamma)} \alpha^2 \|\mub - \mu\|_{1, \delta_\mu} \sp(a_{\mu}^{\mub}(f))  + \frac{1 - \gamma}{\gamma} \sum_{h = 0}^H \gamma^h \nu^\tr P_{\mup}^h [T_{\mu}f - f].
    \end{equation}

    By substituting $H  =0$ and $1,$ the desired results are now easy to see.
\end{proof}

We next prove Theorems~\ref{thm:perf.bds.FA} and \ref{thm:CRPI.FA}.

\begin{proof}[Proof of Theorem~\ref{thm:perf.bds.FA}]
For \(\alpha\in[0,1]\), $\mup = \alpha \mub + (1-\alpha)\mu.$ Since \(T_\mu f \ge f\), Proposition~\ref{prop:perf.Diff.Quad.Bd} shows that
\[
    \nu^\tr Q_{\mup_\alpha} - \nu^\tr f \ge \Psi_1(\alpha) \ge \Psi_0(\alpha)
\]
for every \(\alpha\in[0,1]\). Further, \(\Psi_1\) and \(\Psi_0\), being concave quadratic functions of \(\alpha\), attain their maxima at
\[
    \alS_1
    =
    \frac{
        (1 - \gamma) A_{\mu}^{\mub}(f)
        +
        (1 - \gamma)^2 \nu^\tr a_{\mu}^{\mub}(T_{\mu}f - f)
    }{
        \gamma \|\mub - \mu\|_{1, \delta_{\mu}}\, \sp(a_{\mu}^{\mub}(f))
    }
    \quad \text{and} \quad
    \alS_0
    =
    \frac{
        (1 - \gamma) A_{\mu}^{\mub}(f)
    }{
        \gamma \|\mub - \mu\|_{1, \delta_{\mu}}\, \sp(a_{\mu}^{\mub}(f))
    },
\]
respectively. Since \(\mub\) is greedy with respect to \(f\), we have \(A_{\mu}^{\mub}(f)\ge 0\), and hence \(\alS_0 \ge 0\).

The rest of the proof reduces to locating these maximizers relative to the admissible interval \([0,1]\). Recall from \eqref{e:alS.defn} that the theorem chooses
\[
    \alpha = \min\{1,\alS\},
    \qquad\text{where}\qquad
    \alS =
    \begin{cases}
        \alS_1, & \text{if } \alS_1 > 0,\\
        \alS_0, & \text{otherwise.}
    \end{cases}
\]

We now consider the four possible cases.

\begin{enumerate}
    \item If \(\alS_1 > 1\), then the maximizer of \(\Psi_1\) lies to the right of \(1\). Hence, \(\Psi_1\) is increasing on \([0,1]\), and the best admissible choice is the full greedy step \(\alpha=1\), which is exactly the coefficient selected by the theorem. Therefore,
    \[
        \nu^\tr Q_{\mup} - \nu^\tr f \ge \Psi_1(1).
    \]

    \item If \(\alS_1 \in [0,1]\), then the maximizer of \(\Psi_1\) lies inside the admissible interval. Thus, the best admissible choice is the partial step \(\alpha=\alS_1\), and so
    \[
        \nu^\tr Q_{\mup} - \nu^\tr f \ge \Psi_1(\alS_1).
    \]

    \item If \(\alS_1 < 0\) and \(\alS_0 > 1\), then \(\Psi_1\) is decreasing on \([0,1]\), and hence does not suggest a positive step. In this regime, we instead appeal to the weaker but still valid lower bound \(\Psi_0\). Since its maximizer lies to the right of \(1\), \(\Psi_0\) is increasing on \([0,1]\), so the best admissible choice is again \(\alpha=1\), which is what the theorem selects. Therefore,
    \[
        \nu^\tr Q_{\mup} - \nu^\tr f \ge \Psi_0(1).
    \]

    \item If \(\alS_1 < 0\) and \(\alS_0 \le 1\), then, as in the previous case, we work with \(\Psi_0\). Now its maximizer lies in \([0,1]\), so the best admissible choice is the partial step \(\alpha=\alS_0\). Hence,
    \[
        \nu^\tr Q_{\mup} - \nu^\tr f \ge \Psi_0(\alS_0).
    \]
\end{enumerate}
This completes the proof.
\end{proof}

\begin{proof}[Proof of Theorem~\ref{thm:CRPI.FA}]
As argued in Section~\ref{s:proofs}, we already know that
\[
    T_{\mu_k} f_k \ge f_k \qquad \forall k \ge 0.
\]
Consequently, the first three conclusions of Theorem~\ref{thm:RPI.FA} hold. It remains to prove the fourth one.

Let $f_\infty := \lim_{k\to\infty} f_k,$ and let \(\mu_\infty\) be an arbitrary subsequential limit of \((\mu_k)\); that is, for some subsequence \((k_j)\), we have $\mu_{k_j} \to \mu_\infty.$  Since \(f_k \to f_\infty\), we also have \(f_{k_j} \to f_\infty\). Moreover, $T_{\mu_{k_j}} f_{k_j} \ge f_{k_j},$ for all $j \geq 0.$ Passing to the limit gives $T_{\mu_\infty} f_\infty \ge f_\infty.$ 

We now show that equality must hold in at least one coordinate. Suppose, for contradiction, that
\begin{equation}
\label{eq:strict.CRPI}
    T_{\mu_\infty} f_\infty > f_\infty
    \qquad\text{componentwise.}
\end{equation}
By continuity of \(T_{\mu_\infty}\) and the room-to-improve assumption, it follows by arguing as in the proof of Theorem~\ref{thm:RPI.FA} that there exists \(f \in \cF\) such that $f > f_\infty$ and $T_{\mu_\infty} f > f.$ Since \(\mu_{k_j} \to \mu_\infty\) and \(f\) is fixed, continuity of \(T_\mu f\) in \(\mu\) implies that, for all sufficiently large \(j\), $T_{\mu_{k_j}} f > f.$ Also, because \(f_k \uparrow f_\infty\) componentwise and \(f > f_\infty\), we have \(f > f_{k_j}\) for all large \(j\). Thus, for such \(j\), the vector \(f\) is feasible for the CRPI evaluation step at iteration \(k_j\).

Let \(f_{k_j+1}\) be the optimizer returned at iteration \(k_j\). By optimality,
\begin{equation}
\label{e:f_k_{j+1}.f_k_j.rel1}
    \|f_{k_j+1} - f_{k_j}\| \ge \|f - f_{k_j}\|.
\end{equation}
On the other hand, since \(f_{k_j} \le f_{k_j+1} \le f_\infty\), we have
\[
    0 \le f_{k_j+1} - f_{k_j} \le f_\infty - f_{k_j}.
\]
Further, \(f > f_\infty\) implies
\[
    f - f_{k_j} > f_\infty - f_{k_j} \ge 0.
\]
Hence, by the monotonicity property of the norm, 
\begin{equation}
\label{e:f_k_{j+1}.f_k_j.rel2}
    \|f - f_{k_j}\| > \|f_\infty - f_{k_j}\|
    \quad\text{and}\quad
    \|f_{k_j+1} - f_{k_j}\| \le \|f_\infty - f_{k_j}\|.
\end{equation}
Combining the inequalities in \eqref{e:f_k_{j+1}.f_k_j.rel1} and \eqref{e:f_k_{j+1}.f_k_j.rel2} yields
\[
    \|f_{k_j+1} - f_{k_j}\|
    \ge
    \|f - f_{k_j}\|
    >
    \|f_\infty - f_{k_j}\|
    \ge
    \|f_{k_j+1} - f_{k_j}\|,
\]
a contradiction.

Therefore, \eqref{eq:strict.CRPI} is impossible. This completes the proof.
\end{proof}

We now prove Theorem~\ref{prop:conv_rate}. We begin with a simple inequality.

\begin{lemma}
\label{lem:relu_square}
Let $[t]_+ := \max\{t,0\}$. Then, for any $x \ge 0$ and $c \ge 0$,
\[
    [x-c]_+^2 \;\ge\; \frac{x^2}{4} - c^2.
\]
\end{lemma}

\begin{proof}
We consider two cases.

\paragraph{Case 1: $0 \leq x \le c$.}
In this case, $[x-c]_+ = 0$. Since $0 \leq x \le c$, we have $x^2/4 \le c^2$, and hence
\[
    [x-c]_+^2 = 0 \;\ge\; \frac{x^2}{4} - c^2.
\]

\paragraph{Case 2: $x > c$.}
Here, $[x-c]_+ = x-c$. A direct calculation shows that
\[
    (x-c)^2 - \Big(\frac{x^2}{4} - c^2\Big)
    = \frac{3}{4}x^2 - 2cx + 2c^2
    = \frac{3}{4}\Big(x - \frac{4c}{3}\Big)^2 + \frac{2}{3}c^2 \;\ge\; 0.
\]
Thus, $(x-c)^2 \ge \frac{x^2}{4} - c^2$.

Combining the two cases completes the proof.
\end{proof}

We now show a relation between the relative advantage function estimate and performance difference between the two policies.

\begin{lemma}\label{lem:A.4}
Let $\mu$ and $\mu'$ be arbitrary policies. Further, let $f$ satisfy $T_\mu f \ge f$, and let $\bar\mu$ be greedy with respect to $f$. Define 
\begin{equation}
    \label{e:beta.defn}
    \beta(\mu,\mu') := \min_{(s,a):\, d_{\mu'}(s,a)>0} \frac{d_\mu(s,a)}{d_{\mu'}(s,a)}.
\end{equation}
Then,
\begin{equation}
    \label{e:A.mu.mub.Bd}
    A_\mu^{\bar\mu}(f) \ge \beta(\mu,\mu') \frac{1-\gamma}{\gamma} \Big[\nu^\top Q_{\mu'} - \nu^\top f - \nu^\top [I-\gamma P_{\mu'}]^{-1}[T_\mu f-f] \Big].
\end{equation}
Moreover, if $\|Q_\mu-f\|_\infty \le \epsilon$, then
\begin{equation}
    \label{e:A.mu.mub.eps.Bd}
    A_\mu^{\bar\mu}(f) \ge \beta(\mu,\mu') \frac{1-\gamma}{\gamma} \Big[\nu^\top Q_{\mu'} - \nu^\top Q_\mu - \frac{\epsilon}{1-\gamma}\Big].
\end{equation}
\end{lemma}
\begin{proof}
    The relation in \eqref{e:A.mu.mub.Bd} follows by observing that 
    \begin{align*}
        A_{\mu}^{\mub}(f) \overset{(a)}{=} {} & \sum_{s, a} d_{\mu}(s, a)  \sum_{s', \ap} P(\stp|s,a)  \mub(a'|s') \Bigl(    f(s',a') - \langle \mu(\cdot|s'),\, f(s',\cdot)\rangle \Bigr) \\
        \overset{(b)}{=} {} & \sum_{s, a} d_{\mu}(s, a)  \sum_{s'} P(\stp|s,a)  \max_{\ap} \Bigl(
        f(s',a') - \langle \mu(\cdot|s'),\, f(s',\cdot)\rangle \Bigr) \\
        \overset{(c)}{\geq} {} & \beta(\mu, \mup) \sum_{s, a} d_{\mup}(s, a)  \sum_{s'} P(\stp|s,a)  \max_{\ap} \Bigl(
        f(s',a') - \langle \mu(\cdot|s'),\, f(s',\cdot)\rangle \Bigr) \\
        \overset{(d)}{\geq} {} & \beta(\mu, \mup) \sum_{s, a} d_{\mup}(s, a)  \sum_{s', \ap} P(\stp|s,a) \mup(\ap|s')   \Bigl(
        f(s',a') - \langle \mu(\cdot|s'),\, f(s',\cdot)\rangle \Bigr) \\
        \overset{(e)}{=} {} & \beta(\mu, \mup) d_{\mup}^\tr a_{\mu}^{\mup}(f) \\
        \overset{(f)}{=} {} & \beta(\mu, \mup) \frac{1 - \gamma}{\gamma} \Big[ \nu^\tr Q_{\mup} - \nu^\tr f - \nu^\tr [\bI - \gamma P_{\mup}]^{-1}[T_{\mu}f - f] \Big],
    \end{align*}
where (a) follows from the definition of $A_{\mu}^{\mub}(f),$ (b) follows since $\mub$ is greedy with respect to $f,$ (c) follows from the fact that the inner sum is non-negative, (d) holds since the maximum is larger than any average, (e) follows from the definition of $a_{\mu}^{\mup}(f),$ while (f) follows from \eqref{e:Q_mup.f.exact.relation}.

The relation in \eqref{e:A.mu.mub.eps.Bd} follows by additionally observing that 
\begin{align*}
    \nu^\tr Q_{\mup} - & {} \nu^\tr f  - \nu^\tr [\bI - \gamma P_{\mup}]^{-1}[T_{\mu}f - f] \\
    \overset{(a)}{=} {} & \nu^\tr Q_{\mup} - \nu^\tr Q_{\mu} + \nu^\tr[Q_{\mu} - f]  - \nu^\tr [\bI - \gamma P_{\mup}]^{-1}[T_{\mu}f - f] \\
   \overset{(b)}{\geq} {} & \nu^\tr Q_{\mup} - \nu^\tr Q_{\mu} - \nu^\tr [\bI - \gamma P_{\mup}]^{-1}[T_{\mu}f - f] \\
   \overset{(c)}{\geq} {} & \nu^\tr Q_{\mup} - \nu^\tr Q_{\mu}  - \frac{\epsilon}{1 - \gamma},
\end{align*}
where (a) follows by adding and subtracting $\nu^\tr Q_{\mu},$ (b) follows since $Q_{\mu} \geq f,$ while (c) holds since $(1 - \gamma) \nu^\tr [\bI - \gamma P_{\mup}]^{-1}$ is a distribution, and  since $\|Q_{\mu} - f\|_\infty \leq \epsilon$ implies $Q_{\mu}(s, a) - \epsilon \leq f(s, a) \leq T_{\mu} f(s, a) \leq Q_{\mu}(s, a)$ and, hence, $T_{\mu} f(s, a) -f(s, a) \leq \epsilon.$

The desired results now follow.   %
\end{proof}

\begin{lemma}
\label{lem:A.mu.mub.Bd.al.less.than.1}
    Let $R_{\max} = \max_{s, a}r(s, a).$ Also, let $f$ and $\mu$ be such that  $0 \leq f \leq Q_{\mu}.$ Finally, let $\mub$ be the greedy policy with respect to $f$ and $\mup$ be the mixture policy $\tilde{\alpha} \mub + (1 - \tilde{\alpha}) \mu,$ where
    \[
        \tilde{\alpha} = \frac{(1 - \gamma)^2 A_{\mu}^{\mub}(f)}{2 R_{\max}}.
    \]
    Then, $0 \leq \tilde{\alpha} \leq 1,$ and 
    \[
        \nu^\tr Q_{\mup} - \nu^\tr f \geq \frac{(1 - \gamma) \gamma \Big(A_\mu^{\mub}(f)\Big)^2}{4 R_{\max}}.
    \]
\end{lemma}
\begin{proof}
    First, let $\alpha \in [0, 1]$ be arbitrary, and let $\mup = \alpha \mub + (1 - \alpha) \mu.$ Then, 
    \begin{align}
        \nu^\tr Q_{\mup} - \nu^\tr f \overset{(a)}{\geq} {} & - \frac{ \alpha^2 \gamma^2}{2  (1 - \gamma)^2} \|\mub - \mu\|_{1, \delta_\mu}\ \sp(a_\mu^{\mub} (f)) + \frac{ \alpha\gamma}{1 - \gamma} A_{\mu}^{\mub} (f) + \nu^\tr [T_{\mu}f - f] \\
        \overset{(b)}{\geq} {} & - \frac{ \alpha^2 \gamma^2}{2  (1 - \gamma)^2} \|\mub - \mu\|_{1, \delta_\mu}\ \sp(a_\mu^{\mub} (f)) + \frac{ \alpha\gamma}{1 - \gamma} A_{\mu}^{\mub} (f) \label{e:nu.Q.mup.nu.f.Bd.drop.T.mu.f.term},
    \end{align}
    where (a) follows from \eqref{e:perf.Diff.Quad.Bd}, while (b) follows since $T_{\mu} f \geq f.$

    We now simplify the expression in \eqref{e:nu.Q.mup.nu.f.Bd.drop.T.mu.f.term}. For any policies $\mu, \mup,$ and any state $s \in \cS$,
    \begin{equation}
        \| \mup(\cdot|s)-\mu(\cdot|s) \|_{1} \le 2
    \end{equation}
    and, hence, 
    \begin{equation}
    \label{e:mu.mup.gap.Bd}
        \|\mup - \mu\|_{1, \delta_\mu} \leq 2.
    \end{equation}
    
    Next, we show that, for $f$ and $\mu$ such that $0 \leq f \leq Q_{\mu}$ and $\mub$ greedy with respect to $f,$ we have
    \begin{equation}
    \label{e:sp.a.mu.k.bar.mu.k.Bd}
        \sp \big(a^{\mub}_{\mu}(f)\big) \le \frac{R_{\max}}{1-\gamma}.
    \end{equation}
    Clearly, for any $0 \leq f$, $\mu$, and $s \in \cS$,
    \begin{equation}
        \max f(s, a)-\langle \mu(\cdot|s), f(s,\cdot)\rangle \in [0,\, \sp(f)].
    \end{equation}
    Now, since $\mub$ is greedy with respect to $f,$ it follows from the definition given in  \eqref{e:a.mu.mup.s.a.defn} that
    \[
        a_{\mu}^{\mub}(f)(s, a) \in [0,\, \sp(f)]
    \]
    for every state-action pair $(s, a).$ Therefore,
    \[
        \sp(a_{\mu}^{\mub}(f)) \le \sp(f) \le \|f\|_\infty,
    \]
    where the last inequality uses the fact that $0 \leq f.$
    Finally, since \(0 \le f \le Q_\mu\) and \(\|Q_\mu\|_\infty \leq \dfrac{R_{\max}}{1-\gamma}\), we obtain
    \[
        \|f\|_\infty \le \frac{R_{\max}}{1-\gamma}.
    \]
    Using this relation in the previous display yields \eqref{e:sp.a.mu.k.bar.mu.k.Bd}.
    
    By substituting \eqref{e:mu.mup.gap.Bd} and \eqref{e:sp.a.mu.k.bar.mu.k.Bd} in \eqref{e:nu.Q.mup.nu.f.Bd.drop.T.mu.f.term}, we then get
    \begin{equation}
    \label{e:nu.Q.mup.nu.f.Bd.simplified}
        \nu^\tr Q_{\mup} - \nu^\tr f 
        \geq {} - \frac{ \alpha^2 \gamma^2 R_{\max}}{(1 - \gamma)^3}  + \frac{ \alpha\gamma}{1 - \gamma} A_{\mu}^{\mub} (f).
    \end{equation}
    While the expression on the RHS is maximized at 
    \[
        \alpha = \frac{(1 - \gamma)^2 A_\mu^{\bar{\mu}}(f)}{2 \gamma R_{\max}},
    \]
    this value can potentially be larger than $1.$ Hence, we pick the value stated in the statement, i.e.,
    \[
        \tilde{\alpha} = \frac{(1 - \gamma)^2 A_\mu^{\bar{\mu}}(f)}{2 R_{\max}} \leq \frac{1 - \gamma}{2} \leq 1.
    \]
    Substituting this value of $\tilde{\alpha}$ in \eqref{e:nu.Q.mup.nu.f.Bd.simplified} then gives 
    \[
        \nu^\tr Q_{\mup} - \nu^\tr f 
        \geq {} \frac{(1 - \gamma) \gamma (2 - \gamma) \Big(A_{\mu}^{\mub}(f)\Big)^2}{4 R_{\max}}.
    \]
    Since $2 - \gamma \geq 1,$ we get the desired result. 
\end{proof}

\begin{proof}[Proof of Theorem~\ref{prop:conv_rate}]
We first claim that, for every iteration $k \geq 0,$
\begin{equation}
\label{e:Q.mu.k.Q.S.rel}
    Q_{\mu_k} 
    \leq f_{k + 1} + \epsilon \ones 
    \leq Q_{\mu_k} + \epsilon \ones 
    \leq \QS + \epsilon \ones,
\end{equation}
where $\ones$ is the all-ones vector. The last inequality holds since $Q_\mu \leq \QS$ for any policy $\mu$, and the middle inequality follows from $f_{k+1} \le Q_{\mu_k}$ by Theorem~\ref{thm:CRPI.FA}. 

It remains to verify the first inequality. Let $\hat f := f^{\max}(\mu_k,f_k)$ denote the greatest element of $\mathcal H(\mu_k,f_k)$. By assumption,
\[
    Q_{\mu_k} - \epsilon \ones \le \hat f \le Q_{\mu_k}.
\]
Moreover, $\hat f \in \mathcal H(\mu_k,f_k)$, so $T_{\mu_k}\hat f \ge \hat f \ge f_k$. 

For any $f \in \mathcal H(\mu_k,f_k)$, we have $f \le \hat f$, and hence
\[
    0 \le f - f_k \le \hat f - f_k.
\]
If $f \neq \hat f$, the inequality is strict in at least one coordinate. By strict monotonicity of $\|\cdot\|$,
\[
    \|f - f_k\| < \|\hat f - f_k\|.
\]
Thus, $\hat f$ is the unique maximizer of $\|f - f_k\|$ over $\mathcal H(\mu_k,f_k)$, and hence $f_{k+1} = \hat f$. Therefore,
\[
    f_{k+1} \ge Q_{\mu_k} - \epsilon \ones,
\]
which proves the claim.

\vspace{0.5em}

We now establish a recurrence for $g_k := \nu^\top(Q^\star - Q_{\mu_k})$. We show that
\begin{equation}
\label{eq:gk_recursive_rel}
    g_{k+1} \le g_k - \frac{g_k^2}{C_1} + C_2 \epsilon^2 + \epsilon,
\end{equation}
where $C_1 = \frac{16 R_{\max} \gamma}{(1 - \gamma)^3 \Delta^2}$ and 
$C_2 = \frac{(1 - \gamma) \Delta^2}{4 R_{\max} \gamma}$.

To this end, we first show that
\begin{equation}
\label{eq:conv_rate_relative_adv_sq}
    \big(A_{\mu_k}^{\bar\mu_k}(f_{k+1})\big)^2 
    \ge C_3 g_k^2 - C_4 \epsilon^2,
\end{equation}
where $C_3 = \frac{1}{4}\left(\frac{(1-\gamma)\Delta}{\gamma}\right)^2,$ and $C_4 = \frac{\Delta^2}{\gamma^2}.$

We start with \eqref{eq:conv_rate_relative_adv_sq}. Recall that $[x]_+ := \max\{x,0\}.$ Then, 
\begin{align*}
    A_{\mu_k}^{\bar\mu_k}(f_{k+1}) & \overset{(a)}{\ge} 
    \frac{(1 - \gamma) \beta(\mu_k, \mu^\star)}{\gamma} \left(\nu^\top Q_{\mu^\star} - \nu^\top Q_{\mu_k} - \frac{\epsilon}{1-\gamma}\right)\\
    & \overset{(b)}{\ge} \frac{(1 - \gamma) \Delta}{\gamma} \left(\nu^\top Q_{\mu^\star} - \nu^\top Q_{\mu_k} - \frac{\epsilon}{1-\gamma}\right)\\
    & \overset{(c)}{\ge} \frac{(1-\gamma)\,\Delta}{\gamma} \left( g_k - \frac{\epsilon}{1-\gamma} \right)\\
    & \overset{(d)}{\ge} \frac{(1-\gamma)\,\Delta}{\gamma} \left[ g_k - \frac{\epsilon}{1-\gamma} \right]_+,
\end{align*}
where (a) follows from \eqref{e:A.mu.mub.eps.Bd}, (b) follows since $d_{\mu^\star}(s, a) \leq 1$ and since Assumption~\ref{ass:positivity} implies $d_{\mu_k}(s, a) \geq \Delta$ whenever $d_{\mu^\star}(s, a) > 0,$ (c) follows from the definition of $g_k,$ while (d) follows since $\bar\mu_k$ is greedy with respect to $f_{k+1}$, and thus
$A_{\mu_k}^{\bar\mu_k}(f_{k+1}) \ge 0.$ Therefore, it follows that
\begin{align*}
    \left(A_{\mu_k}^{\bar\mu_k}(f_{k+1})\right)^2 
    & \ge \left(\frac{(1-\gamma)\,\Delta}{\gamma} \right)^2\left[ g_k - \frac{\epsilon}{1-\gamma} \right]_+^2\\
    & \overset{(a)}{\ge} \frac{1}{4} \left(\frac{(1-\gamma)\,\Delta}{\gamma} \right)^2 g_k^2 - \left(\frac{(1-\gamma)\,\Delta}{\gamma} \right)^2 \left(\frac{\epsilon}{1-\gamma}\right)^2, 
\end{align*}
where (a) follows from Lemma \ref{lem:relu_square}. This proves \eqref{eq:conv_rate_relative_adv_sq}. 

Next, using the stepsize choice and Lemma~\ref{lem:A.mu.mub.Bd.al.less.than.1}, we get
\begin{align}
\nu^\top Q_{\mu_{k+1}} - \nu^\top f_{k+1}
&\ge \frac{(1-\gamma)\gamma}{4R_{\max}} 
\big(A_{\mu_k}^{\bar\mu_k}(f_{k+1})\big)^2 \nonumber \\
&\ge \frac{g_k^2}{C_1} - C_2 \epsilon^2 \label{eq:conv_rate_approx_perf_diff}.
\end{align}

Now, 
\begin{align*}
g_{k+1}
&= \nu^\top Q^\star - \nu^\top Q_{\mu_{k+1}} \\
&= g_k - (\nu^\top Q_{\mu_{k+1}} - \nu^\top Q_{\mu_k}) \\
&\le g_k - (\nu^\top Q_{\mu_{k+1}} - \nu^\top f_{k+1}) + \epsilon \\
&\le g_k - \frac{g_k^2}{C_1} + C_2 \epsilon^2 + \epsilon,
\end{align*}
where (a) follows \eqref{e:Q.mu.k.Q.S.rel}, while (b) follows from \eqref{eq:conv_rate_approx_perf_diff}. This establishes \eqref{eq:gk_recursive_rel}.

\vspace{0.5em}

Let $b := C_2 \epsilon^2 + \epsilon$ and define
\[
\delta' := \sqrt{2C_1 b}, 
\qquad 
\delta := \delta' + b.
\]

\textbf{Step 1: Decay outside the $\delta'$-region.}
If $g_k \ge \delta'$, then $b \le g_k^2/(2C_1)$, and hence
\[
g_{k+1} \le g_k - \frac{g_k^2}{2C_1}.
\]
This implies $g_{k+1} \le g_k$, and
\[
\frac{1}{g_{k+1}} - \frac{1}{g_k} 
\ge \frac{1}{2C_1}.
\]
Summing yields
\[
g_k \le \frac{2C_1}{k}.
\]

\textbf{Step 2: Invariance of the $\delta$-region.}
We show that $[0,\delta]$ is invariant.

If $g_k < \delta'$, then
\[
g_{k+1} \le g_k + b < \delta' + b = \delta.
\]

If $\delta' \le g_k \le \delta$, then
\[
g_{k+1} \le g_k - \frac{g_k^2}{C_1} + b
\le g_k - \frac{g_k^2}{2C_1}
\le g_k \le \delta.
\]

Thus, in all cases,
\[
g_k \le \delta \;\Rightarrow\; g_{k+1} \le \delta.
\]

\textbf{Step 3: Global bound.}
Let $\tau := \inf\{k \ge 0 : g_k < \delta'\}$.

If $\tau > N$, then $g_j \ge \delta'$ for all $j \le N$, so
\[
g_N \le \frac{2C_1}{N}.
\]

If $\tau \le N$, then $g_\tau < \delta' < \delta$, and by invariance,
\[
g_N \le \delta.
\]

Therefore, for all $N \ge 1$,
\[
g_N \le \max\left\{\frac{2C_1}{N}, \delta\right\}.
\]

This completes the proof.
\end{proof}

\section{INITIALIZATION UNDER LINEAR FUNCTION APPROXIMATION}
\label{sec:init_linear_fa}

Our theoretical guarantees require the initial value estimate to satisfy
\[
T_{\mu_0} f_0 \;\ge\; f_0.
\]
In this section, we present a simple initialization procedure that ensures this condition under linear FA.

Suppose $\mathcal{F} = \{f_\theta : \theta \in \mathbb{R}^d\}$, where $f_\theta = \Phi \theta$ for a feature matrix $\Phi \in \mathbb{R}^{|\mathcal{S}||\mathcal{A}| \times d}$. Then, for any $(s,a)$,
\[
f_\theta(s,a) = \phi(s,a)^\top \theta,
\]
where $\phi(s,a) \in \mathbb{R}^d$ denotes the feature vector corresponding to $(s,a)$. We assume that the first column of $\Phi$ is a bias feature, i.e.,
\[
\Phi((s,a),1) = 1 \quad \text{for all } (s,a).
\]

Let $\hat f_0 = \Phi \hat \theta$ be any initial estimate. We construct a shifted parameter vector
\[
\theta = \hat \theta + b e_1,
\]
where $e_1 = (1,0,\ldots,0) \in \mathbb{R}^d$ and $b \in \mathbb{R}$ is a scalar. The corresponding function is
\[
f_0 := \Phi \theta = \hat f_0 + b \ones,
\]
where $\ones$ denotes the all-ones vector.

We now choose $b$ so that $T_{\mu_0} f_0 \ge f_0$. Using the affine property of the Bellman operator \citep[Lemma 2.4]{bertsekas1996neuro}, we have
\[
T_{\mu_0}(\hat f_0 + b \ones)(s,a)
= T_{\mu_0} \hat f_0(s,a) + \gamma b.
\]
Thus, the condition $T_{\mu_0} f_0 \ge f_0$ is equivalent to
\[
T_{\mu_0} \hat f_0(s,a) + \gamma b
\;\ge\;
\hat f_0(s,a) + b,
\quad \forall (s,a).
\]
Rearranging, we obtain
\[
T_{\mu_0} \hat f_0(s,a) - \hat f_0(s,a)
\;\ge\;
(1-\gamma)b.
\]
Therefore, it suffices to choose
\[
b \;\le\;
\min_{(s,a)} \frac{T_{\mu_0} \hat f_0(s,a) - \hat f_0(s,a)}{1-\gamma}.
\]

With this choice, the shifted initialization $f_0 = \hat f_0 + b \ones$ satisfies
\[
T_{\mu_0} f_0 \;\ge\; f_0,
\]
as required.

\section{DEEP RL IMPLEMENTATION OF RPI}
\label{sec:deep_rpi}

We now describe how the policy evaluation step of RPI can be implemented
in a deep reinforcement learning setting. The resulting method replaces the
constrained optimization in \eqref{e:GUIDE.PE.Opt} with a stochastic surrogate
objective that can be optimized using standard deep RL infrastructure.

Recall that the policy evaluation step of RPI computes
\[
f_{k+1} \in 
\arg\max_{f\in\mathcal F} \|f-f_k\|
\quad
\text{s.t.}
\quad
T_{\mu_k}f \ge f \ge f_k .
\]
The constraint $T_{\mu_k}f \ge f$ enforces $Q_{\mu_k} \ge f_{k+1}$, while
$f \ge f_k$ ensures monotonic improvement of the value estimates.

\paragraph{Convex surrogate.} We use a
convex surrogate objective that minimizes the Bellman residual while
retaining the same constraints, that is
\[
\min_{f\in\mathcal F} \|T_{\mu_k}f - f\|^2_2
\quad
\text{s.t.}
\quad
T_{\mu_k}f \ge f \ge f_k  .
\]

Intuitively, the original RPI objective maximizes the gap between the
constraints $f \ge f_k$ while  being feasible. In contrast, the surrogate
objective minimizes the gap between $T_{\mu_k}f \ge f$ while being feasible. 


\paragraph{Lagrangian relaxation.}
We now focus on the problem with only the one sided Bellman constraint. We enforce the Bellman inequality constraint $f - T_\mu f \le 0$ using a Lagrangian formulation. The resulting objective takes the form
\begin{equation}
\label{eq:true_lag}
L(f,\lambda)
=
\|T_\mu f - f\|^2_2
+
\lambda^\top (f - T_\mu f),
\quad
\lambda \ge 0.
\end{equation}

In large state–action spaces maintaining a separate multiplier for every constraint is impractical. Instead, we introduce a shared scalar multiplier $\lambda_1$ and approximate the constraint penalty using the ReLU function $[x]_+ = \max(x,0)$. This yields the penalty approximation
\[
\tilde{L}(f,\lambda)
=
\|T_\mu f - f\|^2_2
+
\lambda_1 \, \mathbf{1}^\top
[f - T_\mu f]_+ ,
\]
where $\mathbf{1}$ is an all-ones vector.

The resulting objective closely resembles standard critic losses used in deep reinforcement learning, with the addition of the penalty term enforcing the RPI constraint. As in off-policy deep RL algorithms such as DDPG, we use sampled $\mathcal{B}-$sized minibatch of transitions $(s_i,a_i,r_i,s_i')$ from the replay buffer to approximate the Bellman operator using the target $y_i = r_i + \gamma Q_{\text{target}}(s_i',a_i').$ Substituting this into the relaxed objective yields the critic loss

\[
\mathcal L^{\text{RPI}}_{\text{critic}}
=
\frac{1}{|\mathcal B|}
\sum_i
\Big(
(Q_{\text{curr}}(s_i,a_i) - y_i)^2
+
\lambda_1 [Q_{\text{curr}}(s_i,a_i) - y_i]_+
\Big).
\]

Importantly, this loss
is very similar to existing deep RL critic losses, making it straightforward
to incorporate into standard algorithms.

\paragraph{Adaptive penalty update.}

The multiplier $\lambda_1$ controls the strength of constraint enforcement.
In the true Lagrangian formulation \eqref{eq:true_lag}, the dual variable would be updated using
gradient ascent as
\[
\lambda^{t+1} =
\left[\lambda^t + \eta (f - T_\mu f)\right]_+ .
\]

This update increases the multiplier when the constraint is violated and
decreases it when the constraint is satisfied. Maintaining such updates for
every state–action pair is infeasible in deep RL. Instead, we adopt a
computationally efficient heuristic that mimics the qualitative behavior of
dual ascent.

Specifically, we monitor the average constraint satisfaction rate
$\bar{s}$ computed over a sliding window of recent training iterations.
Let $s_{\text{target}}$ denote a desired target satisfaction rate.
The shared multiplier is then updated as

\[
\lambda_1
\leftarrow
\begin{cases}
\lambda_1 \alpha_\uparrow & \text{if } \bar{s} < s_{\text{target}} \\
\lambda_1 \alpha_\downarrow & \text{otherwise},
\end{cases}
\]

where $\alpha_\uparrow > 1$ and $\alpha_\downarrow < 1$ are scale-up and
scale-down hyperparameters controlling how aggressively the penalty is
adjusted. Finally, we clamp the multiplier to a reasonable range
$\lambda_1 \in [\lambda_{\min},\lambda_{\max}],$ where $\lambda_{\min} \ge 0$.

\paragraph{Integration with deep RL algorithms.}

The proposed critic loss can be directly integrated into standard deep RL
pipelines. In particular, replacing the critic loss in algorithms such as
DDPG with $\mathcal L^{\text{RPI}}_{\text{critic}}$ produces a modified
algorithm that enforces the RPI Bellman inequality during critic training.
We refer to this variant as RPI\textsubscript{DDPG}. This modification requires no changes to the actor update or the overall
training pipeline, making the proposed method easy to incorporate into
existing deep RL implementations.

\paragraph{Implementation details.}

We implement \rpiddpg{} by modifying the publicly available DDPG implementation from the repository accompanying \citep{td3}.\footnote{\url{https://github.com/sfujim/TD3}} 
Specifically, we replace the standard critic loss with $\mathcal L^{\text{RPI}}_{\text{critic}}$ while keeping all other components unchanged.

The additional hyperparameters introduced by our method are as follows: the initial penalty weight $\lambda_1 = 0.01$, window size $n = 100$, target constraint satisfaction rate $s_\text{target}=0.75$, and $\lambda_1$ is updated every $10^4$ iterations. 
Here, the window size denotes the number of recent iterations over which the average constraint satisfaction is computed. 
The penalty weight is adapted multiplicatively with scale factors $2$ and $0.5$, and is constrained to lie in the range $[10^{-3}, 10^{2}]$.

\section{ENVIRONMENT DETAILS}
\label{sec:env_details}

To evaluate the performance of our proposed algorithms, RPI and CRPI, we benchmark them against USPI, AMPI-Q, and CPI on two environments: Inventory Control and Chain Walk. The details of these environments are described below.

\subsection{Inventory Control}
\label{app:inv_ctrl_mdp}

In this work we consider a variant of the classical inventory control problem, modeled as a Markov Decision Process, where the objective is to maximize the expected long-term reward obtained from managing inventory in the presence of stochastic demand.

The inventory system has a fixed maximum capacity of \( M \) units. The state on day \( t \), denoted by \( s_t \in \mathcal{S} = \{0, 1, \dots, M\} \), represents the number of items currently in stock. The action space is \( \mathcal{A} = \{0, 1, \dots, M\} \), where an action \( a_t \in \mathcal{A} \) denotes the number of items ordered at the beginning of day \( t \). Each day yields a reward composed of three components as follows
\begin{itemize}
    \item \textbf{Procurement cost} for the items ordered,
    \item \textbf{Holding cost} for leftover inventory, and
    \item \textbf{Revenue} from items sold.
\end{itemize}

Let \( c \) denote the unit procurement cost, \( h \) the holding cost per unit of unsold inventory, and \( p \) the unit selling price. Let \( d_t \) denote the demand on day \( t \), sampled from a predefined demand distribution. Define the post-order inventory level as
\[
\hat{s}_t = \min(s_t + a_t, M),
\]
which represents the total inventory available for sale on day \( t \) after ordering. The system evolves according to the transition rule
\[
s_{t+1} = \max(\hat{s}_t - d_t, 0),
\]
reflecting the remaining inventory after meeting demand. The immediate reward is given by
\[
R(s_t, a_t, d_t) = p \cdot \min(\hat{s}_t, d_t) - c \cdot a_t - h \cdot \max(\hat{s}_t - d_t, 0),
\]
where the first term captures revenue from sales, the second term penalizes procurement, and the third penalizes excess inventory.


\subsection{Chain Walk}

The chain walk environment is modelled as an $N$-state linear chain with states labeled $1$ to $N$ (with $N=50$ in our experiments). At each step the agent chooses an action ``Left" $(L)$ or ``Right" $(R)$, moving in the intended direction with probability $p$ and in the opposite direction with probability $1-p$ (with $p=0.9$ in our experiments). A reward of $+1$ is granted only upon entering either of the two target states located $N/4$ steps from each end of the chain; all other transitions yield zero reward. No states are terminal, so the chain can be traversed indefinitely. The initial distribution is taken to be uniform over the state-action space.

\section{ADDITIONAL RESULTS ON MUJOCO ENVIRONMENTS}
\label{sec:appendix_all_mujoco_exps}

\begin{figure}[!t]
\centering

\begin{subfigure}{0.95\textwidth}
    \centering
    \includegraphics[width=\linewidth]{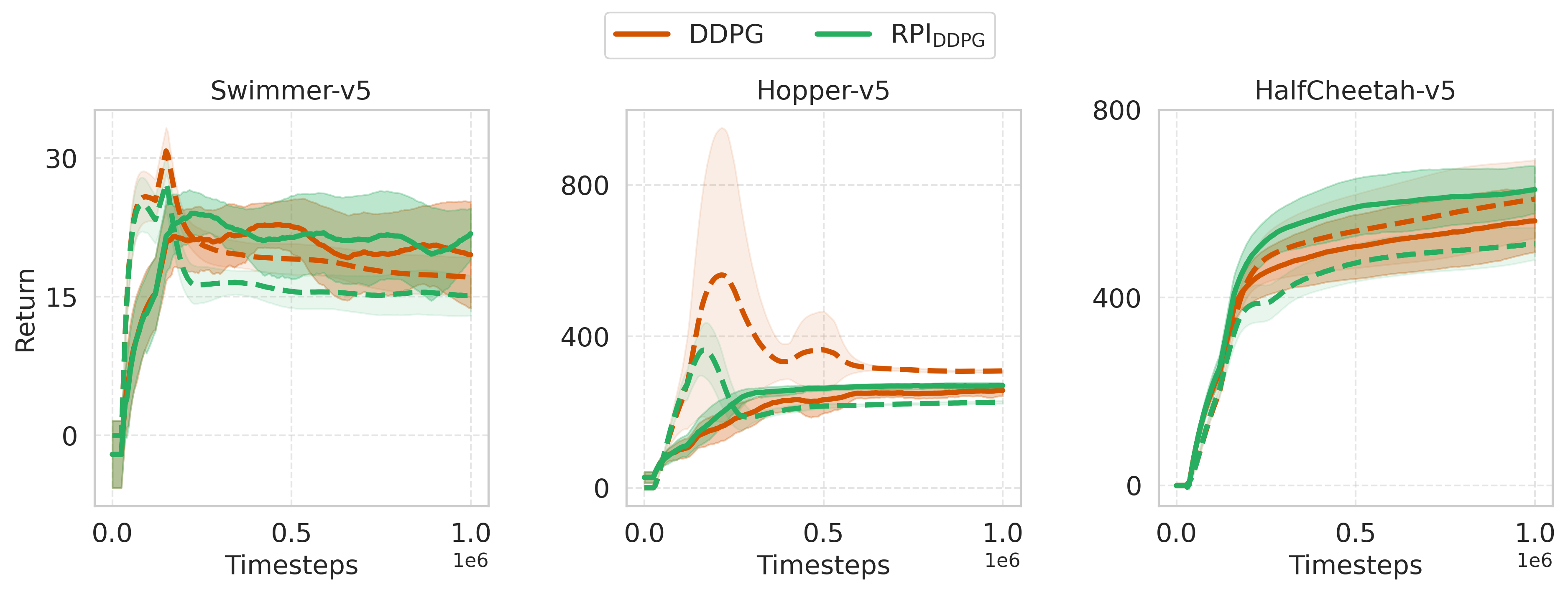}
\end{subfigure}

\vspace{0.6em}

\begin{subfigure}{0.66\textwidth}
    \centering
    \includegraphics[width=\linewidth]{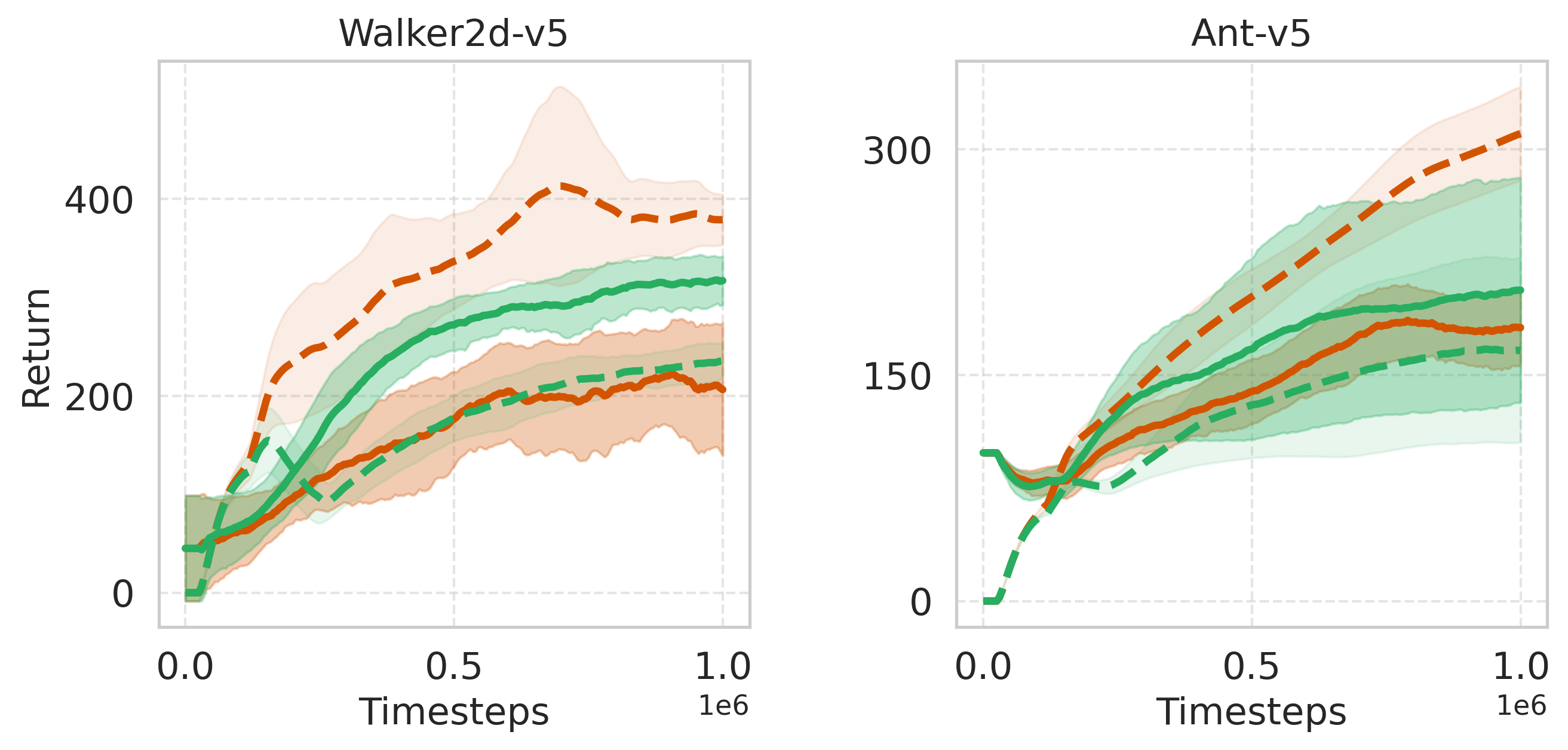}
\end{subfigure}

\vspace{0.8em}

\begin{subfigure}{0.95\textwidth}
\centering
\setlength{\tabcolsep}{5pt}

\begin{tabular}{llccccc}
\toprule
\textbf{Metric} & \textbf{Algorithm} & \textbf{Swimmer} & \textbf{Hopper} & \textbf{HalfCheetah} & \textbf{Walker2d} & \textbf{Ant} \\
\midrule

\multirow{2}{*}{\makecell[l]{Terminal\\Performance}}
& DDPG
& 19.4 $\pm$ 5.8
& 257.6 $\pm$ 9.7
& 564.8 $\pm$ 64.7
& 200.3 $\pm$ 34.1
& 182.8 $\pm$ 7.6 \\

& \rpiddpg
& \textbf{22.3 $\pm$ 1.8}
& \textbf{270.9 $\pm$ 2.0}
& \textbf{632.0 $\pm$ 50.4}
& \textbf{319.9 $\pm$ 17.1}
& \textbf{208.2 $\pm$ 73.5} \\

\midrule

\multirow{2}{*}{AUC ($\times 10^6$)}
& DDPG
& 20.1 $\pm$ 2.7
& 219.3 $\pm$ 10.3
& 483.7 $\pm$ 58.2
& 165.7 $\pm$ 4.8
& 144.4 $\pm$ 8.4 \\

& \rpiddpg
& \textbf{21.0 $\pm$ 2.6}
& \textbf{241.5 $\pm$ 6.5}
& \textbf{550.2 $\pm$ 48.2}
& \textbf{245.0 $\pm$ 10.6}
& \textbf{163.4 $\pm$ 50.5} \\

\bottomrule
\end{tabular}

\end{subfigure}

\caption{
Performance comparison between DDPG and \rpiddpg{} on MuJoCo environments.
\textbf{Top:} Learning curves for Swimmer-v5, Hopper-v5, HalfCheetah-v5, Walker2d-v5, and Ant-v5 (mean $\pm$ 1 std; solid: return, dashed: critic estimate).
\textbf{Bottom:} Terminal performance and AUC (mean $\pm$ std) across all environments.
\textbf{Summary:} \rpiddpg{} maintains lower-bound value estimates, while DDPG often overestimates. \rpiddpg{} matches DDPG on simpler tasks and outperforms DDPG on harder environments.
}

\label{fig:all_mujoco_results}

\end{figure}

This section provides a consolidated view of all MuJoCo environments discussed in Section~\ref{subsec:model_free_experiments}. In particular, we include results for Swimmer-v5 and HalfCheetah-v5, which were omitted from the main paper due to space constraints.

Figure~\ref{fig:all_mujoco_results} presents these results. The plots follow the same format as Figure~\ref{fig:mujoco_results}. Solid curves denote Monte Carlo returns averaged over 25 trajectories, while dashed curves denote the critic estimates for the initial state–action pair, averaged over the same trajectories. All curves are averaged over five random seeds, with shaded regions indicating $\pm 1$ standard deviation.

The results are consistent with the observations in the main text. \rpiddpg{} maintains value estimates that largely remain below the empirical returns after an initial transient phase, while DDPG exhibits noticeable overestimation in most environments. In simpler environments such as Swimmer-v5 and Hopper-v5, both methods achieve similar performance, while in the remaining environments \rpiddpg{} attains higher returns and learns faster. The summary table confirms this trend: \rpiddpg{} achieves higher terminal performance and larger AUC across all environments.

\section{COMPUTATIONAL RESOURCES AND SOLVERS}
\label{sec:comp_res_solvers}
All experiments were run on two high-performance computing machines. The first machine is
powered by an AMD EPYC 7763 64-Core Processor and has 1 TB of RAM, running Ubuntu 20.04.
The second machine is equipped with an AMD EPYC 9554 64-Core Processor, 188 GB of RAM,
and two NVIDIA GeForce RTX 5080 GPUs, running Ubuntu 22.04.

All optimization problems were solved using either the CVXPY \citep{cvxpy1, cvxpy2} or Gurobi \citep{gurobi} Python libraries. We utilized an academic license for the Gurobi optimizer. 
The deep reinforcement learning experiments were implemented in PyTorch \citep{ansel2024pytorch}, with GPU acceleration used where available.

\section{STAGNATION VS.\ DEGRADATION UNDER FUNCTION APPROXIMATION}
\label{sec:limitations}

\begin{figure}[t]
\centering
\begin{tikzpicture}[
    >=stealth,
    every state/.style={circle, draw, minimum size=1cm}
]

\node[state] (s1) at (0, 0) {1};
\node[state] (s2) at (3, 0) {2};
\node[state] (s3) at (1.5, -2) {3}; 

\path[->] (s1) edge[loop above] node {\textsf{stay}, $-1$} (s1);
\path[->] (s2) edge[loop above] node {\textsf{stay}, $-1$} (s2);

\path[->] (s1) edge[bend right=20] node[left=0.1cm] {\textsf{jump}, $-2$} (s3);
\path[->] (s2) edge[bend left=20] node[right=0.1cm] {\textsf{jump}, $-2$} (s3);

\path[->] (s3) edge[loop below] node {\textsf{absorb}, $-10$} (s3);

\end{tikzpicture}

\caption{
Three-state MDP illustrating stagnation vs.\ degradation.
}

\label{fig:rpi_stagnation_mdp}
\end{figure}

Under FA, policy evaluation can fundamentally alter the behavior of policy iteration. Projection-based methods may introduce uncontrolled errors on unvisited state-action pairs, leading to policy degradation. In contrast, RPI enforces Bellman-consistent lower bounds, which prevent such degradation. However, this conservatism can restrict the set of feasible updates: even when the current estimate is feasible, there may be no strictly improving direction within the function class. In such cases, the algorithm stagnates. This section illustrates that such stagnation is not a failure, but a consequence of maintaining reliable value estimates under approximation.

\paragraph{MDP description.}
Consider an MDP with states $S=\{1,2,3\}$ and discount $\gamma=0.9$. The initial-state distribution is uniform over states $1$ and $2$, i.e., $\nu = (0.5, 0.5, 0)$. In states $1$ and $2$, the actions are $\{\textsf{stay},\textsf{jump}\}$, while state $3$ has a single action $\textsf{absorb}$. The agent transitions and rewards are as follows.
\begin{itemize}
    \item \textsf{stay}: reward $-1$, self-loop,
    \item \textsf{jump}: reward $-2$, transitions to state $3$,
    \item \textsf{absorb}: reward $-10$, self-loop in state $3$.
\end{itemize}
A schematic of this MDP is shown in Figure~\ref{fig:rpi_stagnation_mdp}. Let $\pi_0$ (resp. $\pi_1$) choose \textsf{stay} (resp. \textsf{jump}) in states $1,2$.

\paragraph{Exact values.}
We compute $Q_{\pi_0}$ directly from the Bellman equations.

For the absorbing state,
\[
Q_{\pi_0}(3,\textsf{absorb})
= -10 + 0.9\, Q_{\pi_0}(3,\textsf{absorb}),
\]
which gives
\[
Q_{\pi_0}(3,\textsf{absorb}) = -100.
\]

For states $1$ and $2$, under $\pi_0$ the action \textsf{stay} is always taken, so
\[
Q_{\pi_0}(1,\textsf{stay})
= -1 + 0.9\, Q_{\pi_0}(1,\textsf{stay}),
\quad
Q_{\pi_0}(2,\textsf{stay})
= -1 + 0.9\, Q_{\pi_0}(2,\textsf{stay}),
\]
which yields
\[
Q_{\pi_0}(1,\textsf{stay}) = Q_{\pi_0}(2,\textsf{stay}) = -10.
\]

For the action \textsf{jump},
\[
Q_{\pi_0}(1,\textsf{jump})
= -2 + 0.9\, Q_{\pi_0}(3,\textsf{absorb})
= -2 + 0.9(-100) = -92,
\]
and similarly for state $2$.

Thus,
\[
Q_{\pi_0}
=
(-10,-92,-10,-92,-100)^\top,
\]
with coordinates ordered as
\[
(1,\textsf{stay}),\ (1,\textsf{jump}),\ (2,\textsf{stay}),\ (2,\textsf{jump}),\ (3,\textsf{absorb}).
\]

The expected return under the initial state-distribution $\nu = (1/2,1/2,0)$ is
\[
J_\nu(\pi) := \mathbb{E}_{s\sim \nu, a \sim \pi(\cdot|s)}[Q_\pi(s,a)],
\]
and hence
\[
J_\nu(\pi_0) = \tfrac12(-10) + \tfrac12(-10) = -10.
\]

Similarly, under $\pi_1$, both states transition immediately to state $3$, giving
\[
Q_{\pi_1}(1,\textsf{jump}) = Q_{\pi_1}(2,\textsf{jump}) = -92,
\]
while
\[
Q_{\pi_1}(3,\textsf{absorb}) = -100.
\]
Thus,
\[
Q_{\pi_1}
=
(-10,-92,-10,-92,-100)^\top,
\]
with the same coordinate ordering as above, and
\[
J_\nu(\pi_1) = -92 < J_\nu(\pi_0).
\]

\paragraph{Function class.}
We consider a linear FA class of the form
\[
\mathcal F = \{\Phi \theta : \theta \in \mathbb{R}^2\},
\]
where $\theta = (p, 1)^\top$ and the feature matrix $\Phi \in \mathbb{R}^{5 \times 2}$ is given by
\[
\Phi =
\begin{bmatrix}
1 & 0 \\
2 & 15 \\
1 & 0 \\
2 & 15 \\
2 & 15
\end{bmatrix}.
\]
Thus, for any $\theta = (p,1)^\top$, we obtain
\[
f(p) := \Phi \theta = (p,\,2p+15,\,p,\,2p+15,\,2p+15)^\top.
\]

We initialize at
\[
\theta_0 = (-57.5, 1)^\top,
\quad
f_0 = f(-57.5) = \Phi \theta_0 = (-57.5,-100,-57.5,-100,-100)^\top.
\]

\paragraph{Failure of projected/TD-style updates.}
We now analyze the update produced by projected (on-policy) policy evaluation of $\pi_0$. Under $\pi_0$, only the state-action pairs $(1,\textsf{stay})$ and $(2,\textsf{stay})$ are visited. 
Thus, the projected Bellman update minimizes the mean-squared Bellman error over these visited coordinates:
\[
\min_{p \in \mathbb{R}}
\;\;
\mathbb{E}_{(s,a)\sim d_{\pi_0}}
\bigl( f(p)(s,a) - T_{\pi_0} f(p)(s,a) \bigr)^2.
\]

Since $d_{\pi_0}$ is supported only on $(1,\textsf{stay})$ and $(2,\textsf{stay})$, and both coordinates share the same value $f(p)=p$, this reduces to
\[
\min_{p \in \mathbb{R}}
\;\;
\bigl( p - (-1 + 0.9 p) \bigr)^2 = (0.1 p + 1)^2.
\]

The minimizer is obtained by setting the derivative to zero.
\[
0.1 p + 1 = 0 \;\Rightarrow\; p = -10.
\]

Substituting this value,
\[
f(-10) = (-10,-5,-10,-5,-5)^\top.
\]

Thus, the projected update fits the Bellman equation exactly on the visited coordinates:
\[
f(-10)(s,\textsf{stay}) = -10 = Q_{\pi_0}(s,\textsf{stay}),
\]
but distorts the unvisited ones. In particular,
\[
Q_{\pi_0}(1,\textsf{jump}) = -92
\quad \text{while} \quad
f(-10)(1,\textsf{jump}) = -5.
\]

Hence, the \textsf{jump} action is severely overestimated. Greedy improvement therefore selects $\pi_1$, even though $J_\nu(\pi_1) < J_\nu(\pi_0)$. This example shows that projected policy evaluation minimizes Bellman error only on the data distribution $d_{\pi_0}$, and can therefore produce misleading estimates on unvisited state-action pairs, leading to policy degradation.

\paragraph{Behavior of RPI.}

We now verify that $f_0$ satisfies the RPI feasibility condition. 
For each coordinate, we check that $T_{\pi_0} f_0 \ge f_0.$

At states $1$ and $2$, $\pi_0$ selects \textsf{stay}, so
\[
T_{\pi_0} f_0(1,\textsf{stay})
= -1 + 0.9\, f_0(1,\textsf{stay})
= -1 + 0.9(-57.5)
= -52.75 \;\ge\; -57.5,
\]
and similarly for $(2,\textsf{stay})$.

For the \textsf{jump} action,
\[
T_{\pi_0} f_0(1,\textsf{jump})
= -2 + 0.9\, f_0(3,\textsf{absorb})
= -2 + 0.9(-100)
= -92 \;\ge\; -100,
\]
and similarly for $(2,\textsf{jump})$.

Finally, at state $3$,
\[
T_{\pi_0} f_0(3,\textsf{absorb})
= -10 + 0.9(-100)
= -100,
\]
which equals $f_0(3,\textsf{absorb})$.

Thus, $T_{\pi_0} f_0 \ge f_0$ holds coordinate-wise. Moreover, $f_0$ correctly ranks actions at states $1$ and $2$:
\[
f_0(s,\textsf{stay}) = -57.5 \;>\; -100 = f_0(s,\textsf{jump}),
\]
so the greedy policy with respect to $f_0$ coincides with $\pi_0$.

We now show that $f_0$ is the \emph{only} feasible solution. 
Consider any $f(p) \in \mathcal F$. The constraint $T_{\pi_0} f \ge f$ must hold for all coordinates.

Focusing on the absorbing coordinate,
\[
T_{\pi_0} f(3,\textsf{absorb})
= -10 + 0.9\, (2p+15).
\]
The feasibility condition requires
\[
-10 + 0.9(2p+15) \ge 2p+15.
\]
Rearranging,
\[
-10 + 1.8p + 13.5 \ge 2p + 15
\;\;\Longleftrightarrow\;\;
1.8p + 3.5 \ge 2p + 15
\;\;\Longleftrightarrow\;\;
p \le -57.5.
\]

On the other hand, the lower-bound constraint $f \ge f_0$ implies $p \ge -57.5$. 
Therefore, the only feasible value is $p = -57.5$, i.e.,
\[
\{f \in \mathcal F : T_{\pi_0} f \ge f \ge f_0\} = \{f_0\}.
\]

Thus, the RPI update admits no improving direction within $\mathcal F$, and the algorithm returns $f_1 = f_0$. The greedy policy remains $\pi_0$ in subsequent iterations, and the algorithm stagnates at this solution.

\paragraph{Takeaway.}
This example highlights a fundamental trade-off. When the function class poorly represents critical state-action pairs, unconstrained projection can produce overly optimistic estimates and lead to policy degradation. In contrast, RPI enforces lower-bound updates that prevent such degradation, but may stagnate when no feasible improving direction exists. 

In general, intermediate policies generated by RPI may exhibit degradation in their true value. However, the learned estimates remain certified lower bounds, i.e., $f_k \le Q_{\mu_k}$ for all $k$. Thus, even when performance fluctuates, the estimates provide a reliable signal that does not overestimate the true values. Such reliability is not guaranteed in projected Bellman-type evaluation methods, where the estimates can be arbitrarily optimistic.

\end{document}